\documentclass{article}

\PassOptionsToPackage{numbers, compress}{natbib}

\usepackage{arxiv}

\usepackage[utf8]{inputenc} 
\usepackage[T1]{fontenc}    
\usepackage{hyperref}       
\usepackage{url}            
\usepackage{booktabs}       
\usepackage{amsfonts}       
\usepackage{nicefrac}       
\usepackage{microtype}      
\usepackage{xcolor}         

\usepackage{algorithm}
\usepackage{algorithmic}
\usepackage{amsmath}
\usepackage{amssymb}
\usepackage{mathtools}
\mathtoolsset{showonlyrefs}
\usepackage{amsthm}

\usepackage{mathrsfs}
\usepackage{framed}
\usepackage{enumitem}
\usepackage{tcolorbox}

\newcommand{\BibTeX}{\rm B\kern-.05em{\sc i\kern-.025em b}\kern-.08em\TeX}

\theoremstyle{plain}

\newtheorem{assumption}{Assumption}
\newtheorem{definition}{Definition}

\newtheorem{remark}{Remark}
\newtheorem{theorem}{Theorem}
\newtheorem{lemma}{Lemma}

\newtheorem{example}{Example}
\newtheorem{corollary}{Corollary}

\newcommand{\algname}[1]{{\textsf  #1}}

\newcommand{\mbE}{\mathbb E}

\newcommand{\x}{\mathrm{x}}

\newcommand{\LCB}{\texttt{LCB}}
\newcommand{\UCB}{\texttt{UCB}}

\newcommand{\EDRB}{\texttt{ED}$^2$\texttt{RB}}

\newcounter{ineqnumber}

\usepackage{mdframed}
\newmdenv[
    nobreak=true,    
    skipabove=\topskip,
    skipbelow=\topskip,
    linewidth=1pt
]{nframed}

\title{Managing Self-Learning Experts under Per-Round Budget Constraints}

\author{%
  Ilgam Latypov\\
  MSU Institute for Artificial Intelligence\\\
  \texttt{i.latypov@iai.msu.ru}\\ 
  \And
  Alexandra Suvorikova, Alexey Kroshnin\\
   Weierstrass Institute for Applied Analysis and Stochastics\\
  \AND
  Alexander Gasnikov \\
  Innopolis University\\
  \And
  Yuriy Dorn \\
  MSU Institute for Artificial Intelligence\\
}

\begin{document}

\maketitle

\begin{abstract}
        This paper addresses the problem of sequential decision-making under learning budget constraints. Such settings naturally arise in applications like managing a portfolio of bandit or reinforcement learning (RL) algorithms. We propose a novel UCB-type algorithm, M-LCB, designed to manage a pool of $K$ self-learning experts in a stochastic environment while accounting for a limited per-round learning budget $M$. At each round, M-LCB selects one expert to make a decision and at most $M \le K$ experts to learn. For selection, M-LCB uses confidence bounds constructed from limited prior knowledge about the experts (i.e., mild assumptions) and their observed training losses. We derive anytime regret bounds for M-LCB that scale with the individual regrets of the experts. In particular, if each expert has regret $\tilde O(T^\alpha)$ by round $T$, then M-LCB guarantees an overall regret of $\tilde O\left(\sqrt{KT/M} + (K/M)^{1-\alpha}T^\alpha\right)$ relative to the best expert in hindsight. Finally, we demonstrate the applicability of M-LCB using self-learning experts instantiated as (i) parametric models and (ii) bandit algorithms.
\end{abstract}

\section{Introduction}

Consider a production recommendation or ranking service. All candidate models typically score every request in shadow mode, but only one model's output is served to the user, and every model's loss against the realized feedback is logged. Retraining all candidates on every round is infeasible: the training infrastructure supports updates of only a few models per round, so the operator must decide which model's prediction to deploy and which models are worth training further. The same dilemma arises in financial platforms switching between trading strategies and in large-scale online services managing a portfolio of contextual bandits or reinforcement learning algorithms. The common objective is to dynamically select the most accurate model at each step, while managing a limited computational budget for training. This setup falls within the framework of sequential decision-making.

Classical multi-armed bandit (MAB) algorithms \cite{auer2002finite,bubeck2012regret,lattimore2020bandit}, when addressing this problem, usually assume a static or adversarial reward distribution for each arm. Expert algorithms \cite{cesa2006prediction,hazan2016introduction} usually require full feedback  and do not account for how experts' learning rate. Neither approach fully addresses the challenge of managing multiple simultaneously-learning experts within a per-round training budget. The difficulty is that an expert's loss process is endogenous: allocating an update to an expert changes its future performance, so the meta-procedure must simultaneously invest training in experts that may eventually become optimal and deploy an expert whose current performance keeps the immediate regret small. We bridge this gap by proposing a procedure that
unifies prediction with selective training, accounting for a fixed per-round computational budget. Specifically, the contributions of this work are as follows:
\begin{itemize}
    \item \textbf{Novel UCB-Type Meta-Algorithm (M-LCB)}: we propose M-LCB, a novel Upper Confidence Bound (UCB)-type meta-algorithm. It manages a pool of $K$ self-learning experts in a stochastic environment while accounting for a limited per-round learning budget $M (M\le K)$. Training updates are allocated by lower confidence bounds, so that potentially optimal but undertrained experts are not starved of data, while the deployed expert is chosen by upper confidence bounds, so that the immediate loss stays provably small.
    \item \textbf{Computational Efficiency}: we provide a method for constructing confidence bounds directly from realized losses. It is computationally efficient and sidesteps the need for expensive auxiliary optimization.
    \item \textbf{Theoretical analysis}: we estimate the meta-algorithm's performance in terms of the experts' individual convergence rates. For instance, when the experts' regrets are $\tilde O(n^\alpha)$, the overall regret scales as $\tilde O(\sqrt{KT/M} + (K/M)^{1-\alpha}T^\alpha)$. We also prove a matching minimax lower bound, so the interaction between training allocation and deployment identified by M-LCB achieves the optimal regret rate.
    \item \textbf{Extension to Multi-Play Bandits:} we demonstrate that M-LCB extends to the multiple-play bandit setting and controls the top-$M$ subset regret; thus, the same algorithm minimizes both the standard regret and the averaged top-$M$ regret.
\end{itemize}

\subsection{Related works}

\textbf{Self-learning experts (arms).}
The work \cite{dorn2025functional} introduces self-learning arms in the MAB setting: each arm is a black-box parametric function that generates a reward, and its parameter is updated after the arm is played. At each round, the learner selects an arm using a UCB-type index, observes the reward, and then updates the corresponding parameter.

\textbf{Model selection at the meta-level.}
The work \cite{foster2017parameter} introduces a parameter-free aggregation of multiple online learners within the full information framework. \textsc{CORRAL}~\cite{agarwal2017corralling} corrals a pool of bandit algorithms via log-barrier online-mirror descent (OMD) with importance-weighted feedback. The work~\cite{cutkosky2021dynamic} proposes a dynamic balancing meta-algorithm based on known regret rate expressions for the base learners; its regret is defined with respect to the globally optimal action, and only one learner is updated per round. In contrast, our formulation uses \emph{per-expert, prefix-hindsight} guarantees $U_k(T,\delta)$ defined with respect to each expert’s local optimum. The work~\cite{dann2024data} removes the need for known regret rates by
estimating per-learner coefficients online (again, with a single learner updated per round). The closest setting to ours is \cite{pacchiano2020model}, which considers model selection using a \emph{smoothing wrapper}. The regret bounds of these approaches (see Table~\ref{tab:comparison_existing}) coincide with the rate obtained by our method when only one expert can be trained per round. Our algorithm achieves the same order of dependence on $T$ and $\alpha$,
while additionally supporting simultaneous training of up to $M$ adaptive experts with confidence intervals computed directly from realized per-arm losses. Here and below, $c(\delta)$ denotes a problem-dependent confidence scaling, which is polylogarithmic in $1/\delta$ for standard experts.

\textbf{MABs with updates of multiple arms.}
In this setting, the meta-procedure can update or observe several arms per round. Several works consider the \emph{adversarial} case.
\cite{seldin2014prediction} study prediction with limited advice (query at most $M$ arms) and obtain the regret bound $\tilde{O}\Big(\sqrt{\tfrac{K T \log K}{M}}\Big)$.
\cite{yun2018multi} presents the minimax-optimal regret $\tilde O\Big(\max\{\sqrt{KT/M},\,\sqrt{T\log K}\}\Big)$, matching the lower bounds from \cite{seldin2014prediction}. However, these works assume non-learning arms (experts). In the \emph{stochastic} case, \cite{yun2018multi} considers the multi-armed bandit with additional observations and proposes the \emph{KL-UCB-AO} algorithm with asymptotically optimal \emph{logarithmic} regret; however, its applicability is limited by the Kullback-Leibler-based selection rule.

\textbf{Partial monitoring.}
Unlike standard partial monitoring \cite{cesa2006prediction}, where losses are estimated from signals, our model observes realized losses for selected experts and relies on the learnability Asm.~\ref{assumption:expert_regret}, absent from the standard formulation. PM algorithms satisfying this assumption can nonetheless be incorporated as base experts in our framework.


\textbf{Multiple-play multi-armed bandits.}
In the \emph{multiple-play} setting (see \cite{agrawal1990multi,uchiya2010algorithms}),
a meta-procedure selects $M$ arms per round and observes semi-bandit feedback.
UCB-based algorithms for combinatorial bandits~\cite{kveton2015tight}
achieve $\tilde O(\sqrt{KT/M})$ regret under stochastic rewards,
providing a baseline for subset-level performance analysis. These results serve as a benchmark for multiple-play MAB extension of \algname{M-LCB}. Table~\ref{tab:comparison_existing} compares the problem setups.

\begin{table*}[!h]
\centering
\small
\caption{
Assuming $U_k(T,\delta)=O(T^{\alpha}c(\delta))$ with $\alpha \in [0,1]$ and poly-logarithmic $c(\delta)$, we compare regret rates up to logarithmic factors ($\checkmark$ indicates supported properties). Regret follows Section~\ref{sec:multiple_play} for multiple-play bandits, specific conventions for \cite{seldin2014prediction, yun2018multi}, and Section~\ref{sec:global-regret} otherwise.
\textbf{Our \algname{M-LCB} algorithm attains optimal rates for both regret definitions.}}

\label{tab:comparison_existing}
\begin{tabular}{p{3cm}lp{1cm} p{1cm} p{5.0cm}}
\toprule
\textbf{Algorithm / Reference} & \textbf{Learnable} & \textbf{Multi-arm} & \textbf{Multiple-play} & \textbf{Regret rate (up to logs)} \\
\midrule
\textsc{CORRAL} + smoothing wrapper~\cite{pacchiano2020model}
& $\checkmark$ & $\times$ & $\times$ &
$\tilde O\!\left(\sqrt{KT} + K^{\alpha}T^{1-\alpha} + K^{1-\alpha}T^{\alpha}c(\delta)\right)$ \\[1pt]

\textsc{EXP3.P} + smoothing wrapper~\cite{pacchiano2020model}
& $\checkmark$ & $\times$ & $\times$ &
$\tilde O\!\left(\sqrt{KT} + K^{\frac{1-\alpha}{2-\alpha}}T^{\frac{1}{2-\alpha}}c(\delta)^{\frac{1}{2-\alpha}}\right)$ \\[1pt]

Dynamic Balancing~\cite{cutkosky2021dynamic}
& $\checkmark$ & $\times$ & $\times$ &
$\tilde O\!\left(\sqrt{KT} + K^{1-\alpha}T^{\alpha}c(\delta)\right)$ \\[1pt]

Prediction with Limited Advice~\cite{seldin2014prediction}
& $\times$ & $\checkmark$ & $\times$ &
$\tilde O\!\left(\sqrt{\tfrac{KT\log K}{M}}\right)$ \\[1pt]

H-INF~\cite{yun2018multi}
& $\times$ & $\checkmark$ & $\times$ &
$\tilde O\!\left(\max\{\sqrt{KT/M},\,\sqrt{T\log K}\}\right)$ \\[1pt]

UCB-based multiple-play bandits~\cite{kveton2015tight}
& $\times$ & $\times$ & $\checkmark$ &
$\tilde O\!\left(\sqrt{KT/M}\right)$ $\alpha = \frac{1}{2}$ \\[1pt]

\textbf{\algname{M-LCB} [this work]}
& $\checkmark$ & $\checkmark$ & $\checkmark$ &
$\tilde O\!\left(\sqrt{\tfrac{KT}{M}} + (K/M)^{1-\alpha}T^{\alpha}c(\delta)\right)$ \\
\bottomrule
\end{tabular}
\vspace{-1em}
\end{table*}

\section{Problem Setup}
\label{sec:setup}
In this section, we formalize the interaction between the meta-procedure $\mathcal{P}$ and the pool of trainable experts. Section~\ref{sec:losses} details the interaction protocol. Section~\ref{sec:global-regret} defines the objective in terms of cumulative regret. Section~\ref{sec:arm-dynamics} characterizes the learning dynamics of the experts, while Section~\ref{sec:examples} illustrates our framework through concrete examples in parametric learning and bandit selection.

\subsection{Procedure protocol}
\label{sec:losses}
\vspace{-0.75em}

We consider a stochastic online decision-making problem with $K$ experts and a per-round training budget $M$. Let $\mathbf{U}$ be a decision space and $\mathbf{E}$ be an outcome space. The environment is characterized by a loss function $\ell: \mathbf{U} \times \mathbf{E} \to \mathbb{R}_+$. Each expert $k \in [K]$ operates on a specific subspace $\mathbf{U}_k \subseteq \mathbf{U}$ and generates a prediction $u_k^t \in \mathbf{U}_k$ at round $t$.

The meta-procedure $\mathcal{P}$ manages this pool of experts. At each round $t$, $\mathcal{P}$ allocates the training budget to a subset $S_t \subseteq [K]$ (where $|S_t| \le M$) and selects a global decision $u^t \in \mathbf{U}$ by adopting the advice of a chosen expert. The environment then reveals a random outcome $\xi^t$ taking values in $\mathbf{E}$, and $\mathcal{P}$ incurs the global loss $\ell(u^t, \xi^t)$.

Crucially, the expert learning is decoupled from the meta-procedure $\mathcal{P}$'s immediate performance. Experts update their internal parameters based solely on their own prediction loss, $\ell(u_k^t, \xi^t)$, and only when selected in the training set $S_t$. This creates a dual role for $\mathcal{P}$: it explores by allocating the training budget (choosing $S_t$ to improve specific experts) and exploits by aggregating advice from the selected expert to form the final prediction $u^t$ (see Section~\ref{sec:algorithm}). This interaction mirrors the \emph{prediction with limited advice} framework \cite{seldin2014prediction}. The box below summarizes the protocol.
\begin{nframed}
\noindent
\textbf{Protocol: Decision making with trainable experts}\\[0.25em]
For $t=1,2,\dots$:
\begin{enumerate}[leftmargin=2em]
    \item Choose $S_t\subseteq[K]$ with $|S_t|\le M$ and decision $u^t\in\mathbf U$.
    \item Observe $\xi^t\sim D$ and incur loss $\ell(u^t,\xi^t)$.
    \item For each $k\in S_t$, compute $\ell(u_k^t,\xi^t)$ and update expert $k$.
\end{enumerate}
\end{nframed}

The key feature of this protocol is that the environment's outcome $\xi^t$ is revealed to the meta-procedure, so the loss $\ell(u_k^t,\xi^t)$ of every trained expert $k\in S_t$ is computable even if that expert was not deployed. This is natural whenever $\xi^t$ is a complete realization of the round (e.g., a revealed label, an observed market outcome, or a simulator state), or whenever the trained experts' sampled actions can be probed; Section~\ref{sec:feedback-domains} details such domains.

\subsection{Where does the feedback model hold?}
\label{sec:feedback-domains}
We highlight two domains in which the per-round feedback of the protocol above is directly available; further domains (simulator and digital-twin training, network routing with probing, off-policy estimation) are detailed in Appendix~\ref{sup:feedback_domains}, together with standard expert classes and their anytime rates $U_k(n,\delta)$ (see also Table~\ref{tab:convergence_examples}).

\emph{Full-feedback online supervised learning.} The outcome is a pair $\xi^t=(x_t,y_t)$, and once the label $y_t$ is revealed, the loss of every model is computable, whether or not its prediction was acted upon. This is the standard full-information online learning protocol \cite{hazan2016introduction}, and it covers, e.g., clinical decision support, where the physician follows one recommendation but all candidate models are scored retrospectively against the realized outcome; our real-data experiments (Appendix~\ref{sup:real_data}) have exactly this form.

\emph{Shadow deployment and A/B testing.} Production recommender and ranking systems routinely score every request with all candidate models in shadow mode; only one model's output is served, but every model's loss against the realized user feedback is logged. The budget $M$ is the training (GPU) budget deciding which shadowed models are updated.

\subsection{Regret}
\label{sec:global-regret}
Let $L(u) := \mathbb{E}_{\xi \sim \mathcal{D}}[\ell(u,\xi)]$ be the expected loss of a decision $u \in \mathbf{U}$. The optimal achievable expected loss $L_k^\star$ of $k$-th expert, and the global minimum $L^\star$ are,

\begin{equation}
\label{def:expeted_smallest_loss}
L_k^\star := \min_{u \in \mathbf{U}_k} L(u),
\qquad
L^\star := \min_{k \in [K]} L_k^\star.
\end{equation}
The goal of $\mathcal{P}$ is to minimize the cumulative regret relative to the best expert in the pool:
\begin{equation}
\label{eq:regret}
\mathrm{Reg}(T) := \sum_{t=1}^T \ell(u^t, \xi^t) - T \cdot L^\star.
\end{equation}

This definition follows the standard regret formulation in the stochastic multi-armed bandit literature. We operate under the following standard assumptions on the environment:
\begin{assumption}[Stochastic bounded losses]
\label{assumption:stochastic_losses}
At each round, the environment draws $\xi^t \overset{\mathrm{iid}}{\sim} D$ from an unknown distribution $D$ supported on $\mathbf{E}$.
The loss is $\ell:\mathbf{U}\times \mathbf{E}\to [0,1].$
\end{assumption}

\begin{remark}[On the bounded loss]
\label{rem:bounded-loss}
The boundedness is used only in the meta-level concentration step. For unbounded losses with finite variance, robust mean estimators (cf. \cite{bubeck2013bandits}) keep the same width up to log factors, so for a fixed $U_k$ the meta-guarantees survive at the cost of logs; a slower expert-level rate $U_k$ for heavy-tailed subproblems is a separate matter (see Table~\ref{tab:convergence_examples}).
\end{remark}

\subsection{Self-learning experts}
\label{sec:arm-dynamics}
Each expert $k$ is an online learning algorithm that improves as it is selected for training. Let $I_k(t) := \{\tau < t : k \in S_\tau\}$ denote the set of rounds up to time $t$ where expert $k$ was selected for training. We define $n_k^t := |I_k(t)|$ as the total number of such updates.
We characterize the performance of these experts through an \textit{anytime regret bound}. This ensures that as an expert is trained, its decisions $u_k^t$ converge toward its optimum $L_k^\star$. The running loss and regret are
\begin{equation}
\label{def:experts_running_loss}
L_{k}(t) := \frac{1}{n^t_k} \sum_{\tau \in I_{k}(t)} \ell(u_k^\tau,\xi^\tau),
\quad
R_k(t) := n^t_k \left( L_{k}(t) - L_k^\star\right).
\end{equation}
\begin{assumption}[Anytime expert regret bound]
\label{assumption:expert_regret}
For any $\delta\in(0,1)$, with probability at least $1-\delta$, for all $t\ge 1$, $R_k(t) \le U_k(n_k^t,\delta),$
where $U_k(\cdot,\delta)$ is non-decreasing. The function $U_k$ is deterministic, data-independent, and known to the meta-procedure in advance.
\end{assumption}
Asm.~\ref{assumption:expert_regret} aligns with standard performance metrics in multi-armed bandits. We further show that regret in online optimization satisfies it. Note that the assumption is not vacuous: it requires the \emph{random} regret $R_k(t)$ to be dominated, uniformly over trajectories and simultaneously over time, by a known computable envelope — in other words, that the expert is learnable at a known rate. This is the same style of assumption as in smoothing-wrapper model selection~\cite{pacchiano2020model}.

We deliberately treat the expert's update as a black box: an \emph{update} of expert $k$ is any measurable map from its loss history $\{(u_k^\tau,\ell(u_k^\tau,\xi^\tau)):\tau\in I_k(t)\}$ to its next decision $u_k^{t+1}$. No structure is imposed on this map beyond the anytime guarantee of Asm.~\ref{assumption:expert_regret}; Sections~\ref{sec:examples} and~\ref{sup:bandit_application} instantiate it concretely (a projected gradient step for OCO experts, index recomputation for bandit experts).

\subsection{Examples}
\label{sec:examples}
We instantiate the framework with two standard expert classes, described in detail in Appendix~\ref{sup:usage_examples}. In \emph{online parametric learning}, each expert is an Online Convex Optimization algorithm (e.g., OGD) over its own parameter space; such experts satisfy Asm.~\ref{assumption:expert_regret} with $U_k(n, \delta) = O\big(GD_k\sqrt{n} + \sqrt{n \log(n/\delta)}\big)$. In the \emph{bandit} instantiation, each expert is a stochastic bandit algorithm (e.g., \algname{UCB}) over its own action set, with $U_k(n, \delta) = O(\sqrt{m_k n \log(n/\delta)})$; here the limited-advice feedback consists of the sampled-action losses of the trained experts. Concrete domains in which the feedback model holds for these expert classes are discussed in Section~\ref{sec:feedback-domains} and Appendix~\ref{sup:feedback_domains}.

\section{The \algname{M-LCB} algorithm}
\label{sec:algorithm}

We balance optimistic learning with conservative deployment: we allocate the training budget to potentially optimal experts with the lowest lower confidence bounds ($LCB_k$), while selecting the deployed advisor using the lowest upper confidence bounds ($UCB_k$) to minimize immediate loss. This prevents over-exploiting uncertain experts.

\begin{definition}[UCB and LCB]
\label{def:ucb_lcb}
Fix $k\in[K]$ and let
\begin{align*}
\LCB_k(t,\delta)
:= L_{k}(t)- \tfrac{U_k\left(n^t_k,\delta_{\mathrm{arm}}\right)}{n^t_k}, \quad
\UCB_k(t,\delta)
:= L_{k}(t) + H(n^t_k,\delta_{n^t_k}),
\end{align*}
where $\delta_{\mathrm{arm}} = \tfrac{\delta}{2K}$,  $\delta_n = \tfrac{\delta}{3.5Kn^2}$ and $H(n,\delta) = \sqrt{\tfrac{2\log(1/\delta)}{n}}$,
with $U_k(\cdot)$ being the regret bound from Asm.~\ref{assumption:expert_regret} and $L_k(t)$ being an accumulated loss defined in \eqref{def:experts_running_loss}.
\end{definition}

Notably, if our problem admits expert regret of the form~\eqref{eq:regret_oco} we propose tighter confidence bounds derived using the proof technique based on self-normalized processes (see Section~\ref{sec:oco_style_bounds}).

At round $t$, we compute $\LCB_k(t,\delta)$ and $\UCB_k(t,\delta)$ for each expert $k$. We then determine the training subset $S_t$ and the advisor $i_t$ using the following rules:
\begin{align*}
S_t &:= \arg\min_{\substack{S\subseteq[K]\\ |S|=M}} \sum_{k\in S}\LCB_k(t,\delta), ~~
i_t = \arg\min_{k\in S_t}\UCB_k(t,\delta).
\end{align*}

The two rules formalize a separation of roles — optimism for training, conservatism for execution. In practice, choosing the advisor models deploying a specific model as the predictor for the current round, while the remaining trained experts keep learning. Training is allocated by LCBs: a small $\LCB_k$ means the expert's optimal loss $L_k^\star$ might be small, so the expert deserves budget despite current uncertainty. The deployed advisor is chosen by UCBs: since $\UCB_k$ provably upper-bounds the expected loss of the smoothed decision (Lemma~\ref{prop:anytime-ci-clear}), minimizing it keeps the immediate loss small even in the worst case.

To address the lack of pointwise convergence in stochastic algorithms \cite{cesa2006prediction, shalev2012online, pacchiano2020model}, we use smoothing wrappers \cite{orabona2019modern, pacchiano2020model} to robustly construct the decision $u^t$ of expert $i_t$, ensuring $\mathcal{P}$ satisfies Asm.~\ref{assumption:wrapper}. Unlike adversarial aggregation (even when $M=K$), our confidence bounds estimate optimal values by directly exploiting the base experts' structural learnability guarantees.


\begin{assumption}[Smoothing wrapper]
\label{assumption:wrapper}
A smoothing wrapper is a possibly randomized mapping $\upsilon$ that takes a finite sequence of decisions
$H=\{u^\tau:\tau\in Z\}\subseteq \mathbf U$ with $|Z|\ge 1$ and returns a \emph{probability distribution} $\upsilon(H)$ over $\mathbf U$, s.t.
\begin{equation*}
\mathbb E_{u\sim \upsilon(H)}\bigl[L(u)\mid H\bigr]\ \le\ \frac{1}{|Z|}\sum_{\tau\in Z} L(u^\tau),
\qquad
L(u):=\mathbb E_{\xi\sim D}[\ell(u,\xi)].
\end{equation*}
A deterministic wrapper returning $u=\upsilon(H)\in\mathbf U$ is the special case where $\upsilon(H)$ is a point mass.
\end{assumption}

\begin{example}
\label{example:wrapper}
Assume $\mathbf U$ is convex and $\ell(\cdot,\xi)$ is convex in $u$ for all $\xi$. A natural choice is the average $u = \frac{1}{|Z|}\sum_{\tau\in Z} u^\tau,$
which satisfies Asm.~\ref{assumption:wrapper} by Jensen's inequality. The randomized form also covers non-convex decision spaces: uniformly sampling $u$ from the history $\{u^\tau:\tau\in Z\}$ \cite{pacchiano2020model} satisfies the inequality in conditional expectation without convexity.
\end{example}

Alg.~\ref{algorithm:M_LCB} outlines the procedure. Initializing each expert once before the main loop ensures well-defined confidence bounds, costing at most $K$ updates without affecting asymptotic regret. The \textit{Update} step handles both internal learning and bookkeeping (e.g., counters, accumulated loss).


\begin{algorithm}[ht!]
\caption{\algname{M-LCB}}
\label{algorithm:M_LCB}
\begin{algorithmic}[1]
\STATE \textbf{Input:} experts $\{\mathbf{U}_k\}_{k=1}^K$, per-round budget $M$, confidence parameter $\delta$, smoothing wrapper $\upsilon$ (see Assumption \ref{assumption:wrapper}).

\STATE \textbf{Output:} (i) sequence of advices $\{u^t\}_{t=1}^T$; (ii) experts trained at each round $\{S_t\}_{t=1}^T$

\STATE Initialize each expert with $\delta_{\mathrm{arm}} = \delta/(2K)$ (see Def.~\ref{def:ucb_lcb})
\STATE Warm-start each expert once so that $n_k^t\ge 1$ for all $k$ before the main selection rule is used.
\FOR{$t=1,2,\dots,T$}
    \FOR{each $k \in [K]$}
        \STATE Compute $\LCB_k(t,\delta), \UCB_k(t,\delta)$
    \ENDFOR
    \STATE $S_t \gets \arg\min_{S\subseteq[K],|S|=M}\sum_{k\in S}\LCB_k(t,\delta)$
    \STATE $i_t \gets \arg\min_{k\in S_t}\UCB_k(t,\delta)$
    \STATE $u^t \gets \upsilon(\{u_{i_t}^\tau: \tau \in I_{i_t}(t)\})$ \hfill // smoothing over $i_t$
    \STATE Play $u^t$, suffer $\ell(u^t,\xi^t)$
    \FOR{each $k \in S_t$}
        \STATE Observe $\ell(u_k^t, \xi^t)$.
        \STATE Update expert $k$.
    \ENDFOR
\ENDFOR
\end{algorithmic}
\end{algorithm}

\section{Regret bounds}
\label{sec:theoretical_analysis}
Let \emph{concentration event $\mathcal{E}_{\delta}$} ensure that all confidence bounds hold simultaneously for every expert and every time step (for details see Lemma~\ref{prop:anytime-ci-clear}). First, we establish high-probability bounds on the pseudo-regret and the realized regret.
\begin{lemma}[Regret bounds]
\label{lem:regret_bounds}
Let the pseudo-regret be
$\widetilde{\mathrm{Reg}}(T)
:= \mathbb E_{\upsilon}\left[\sum_{t=1}^T L(u^t)\right]  - T \cdot L^\star,$
where the expectation is over the possible randomization of the smoothing wrapper.
With $L(\cdot) := \mathbb{E}_{\xi \sim \mathcal{D}}[\ell(\cdot,\xi)]$, fix $\delta \in (0,1)$ and let Assumptions~\ref{assumption:stochastic_losses}-\ref{assumption:wrapper} hold.
Then for the
pseudo-regret of Alg.~\ref{algorithm:M_LCB} it holds on the concentration event $\mathcal{E}_{\delta}$ that for all $T\ge 1: \widetilde{\mathrm{Reg}}(T) \le \mathbb E_{\upsilon}[\Delta(T)]$, where
\begin{align}
\label{eq:bound_bound_bound}
    \Delta(T) :=
     \sum_{\tau \in I_{k^{\star}}(T)} \bigl[\UCB_{k^\star}(\tau, \delta)-\LCB_{k^\star}(\tau, \delta)\bigr]
    + \frac{1}{M}\sum_{k=1}^K \sum_{{\tau \in I_{k}(T)}} \bigl[\UCB_k(\tau, \delta)-\LCB_k(\tau, \delta)\bigr],
\end{align}

    where $k^\star$ is the index of the best expert.
Moreover, with probability at least $1-2\delta$, it holds that for all $T\ge 1$, $\mathrm{Reg}(T) \le \mathbb E_{\upsilon}[\Delta(T)] + O(\sqrt{T}).$
\end{lemma}

We specify this result for the case when the experts' regrets are $\tilde O(n^\alpha)$.

\begin{theorem}[Convergence rates]
\label{thm:gen-rate}
Let $\alpha, \delta \in (0, 1)$ and each expert $k$ satisfy Asm.~\ref{assumption:expert_regret}
with $U_k(t,\delta) =O\big(t^{\alpha}c(\delta)\big)$, where $c(\delta)$ is a problem-dependent confidence scaling; it is polylogarithmic in $1/\delta$ for all experts in Table~\ref{tab:convergence_examples}.
Then, with probability at least $1 - \delta$ the regret of Alg.~\ref{algorithm:M_LCB} is bounded for all $T\ge1$ as
\begin{equation*}
\mathrm{Reg}(T)
= O\left(\sqrt{\frac{K T}{M}\log(\frac{KT}{\delta})} + \Bigl(\tfrac{K}{M}\Bigr)^{1-\alpha} T^{\alpha} c(\delta)\right).
\end{equation*}
\end{theorem}
We also get lower bounds.
\begin{definition}[Stochastic tasks with $\alpha$–regret lower bound]
\label{def:alpha-class}
Let $\alpha \in [0,1]$.
A family $\mathcal{F}_\alpha$ of stochastic online learning problems admits an \emph{$\alpha$–regret lower bound}
if for every learning algorithm $\mathscr{A}$ and every $T \ge 1$,
$\sup_{f \in \mathcal{F}_\alpha} \mathbb{E}\bigl[R_{\mathscr{A}}(T)\bigr] \ge c\,T^{\alpha}$
for some constant $c > 0$ independent of\, $T$ and $\mathscr{A}$.
\end{definition}

\begin{theorem}[Lower bound]
\label{thm:mt-alpha}
Consider the single-deployment setting of Section~\ref{sec:losses} with $K$ experts, horizon $T$, and a per-round budget $M$. Fix $\alpha\in[0,1]$. There exists a class $\mathcal F_\alpha$ satisfying Def.~\ref{def:alpha-class}, s.t. if each
expert $k \in [K]$ solves a problem $f_k \in \mathcal F_\alpha$, then, for sufficiently small $\sqrt{\frac{K\log K}{MT}}$ and for any algorithm $\mathscr{A}_k$ and meta-procedure $\mathcal{P}$
\begin{equation}
\label{eq:mt-alpha-internal_appendix}
\sup_{f_{k} \in \mathcal{F}_{\alpha},\,k \in [K]} \mathbb{E}\,\mathrm{Reg}(T) \ge c_1\,\sqrt{\frac{K\,T}{M}}
\;+\; c_2\,T^\alpha \left(\frac{K}{M} \right)^{1-\alpha},
\end{equation}
where $c_1,c_2>0$ are constants depending only on $\alpha$.
\end{theorem}
\begin{proof}
See Appendix~\ref{sec:lower-bounds} (restated there as Theorem~\ref{thm:mt-alpha_appendix}).
\end{proof}
This establishes a fundamental lower bound for managing learnable experts under a budget, and Theorem~\ref{thm:gen-rate} matches it up to logarithmic factors (proof of Theorem~\ref{thm:mt-alpha} in Appendix).

\noindent
\textbf{Proof idea.}
We use heavy-tailed MABs \cite{bubeck2013bandits} to construct $\mathcal{F}_{\alpha}$, with key ingredients from \cite{bubeck2010bandits, seldin2014prediction}. In the construction, deploying any expert other than the optimal one costs $\varepsilon+\Delta$ per round, and the heavy-tailed noise makes the candidates statistically hard to distinguish: identifying the optimal expert costs $T^\alpha(K/M)^{1-\alpha}$, while the cross-expert exploration costs $\sqrt{KT/M}$.

\subsection{A connection to Multiple-Play Bandits}
\label{sec:multiple_play}
\algname{M-LCB} extends to the multiple-play bandit setting \cite{agrawal1990multi, uchiya2010algorithms} by deploying all $M$ experts in $S_t$ and incurring their mean loss. Performance is assessed against the optimal $M$-expert subset; notably, in this setting, UCB-based algorithms \cite{kveton2015tight} achieve $\tilde O(\sqrt{KT/M})$ convergence rates under fixed stochastic rewards.

\begin{lemma}[Top-$M$ experts regret]
\label{prop:m_mean_regret}
Fix $\delta \in (0,1)$,
and assume the conditions of  Lemma~\ref{lem:regret_bounds} hold.
Let $\overline{L}^\star = \min_{S\subseteq[K],\,|S|=M}\frac{1}{M}\sum_{k\in S}L_k^\star$ be the mean optimal loss among the $M$ best experts. Let $\mathrm{Reg}_M(T) := \sum_{t=1}^{T}
\left[\frac{1}{M}\sum_{k \in S_t} L^\star_k - \overline{L}^\star \right].$ The below bound holds with probability at least $1-\delta$, $$     \mathrm{Reg}_M(T) \le \frac{1}{M}\sum_{k=1}^K \sum_{{\tau \in I_{k}(T)}} \bigl[\UCB_k(\tau, \delta)-\LCB_k(\tau, \delta)\bigr].$$
\end{lemma}

The proof is presented in Appendix \ref{proof:m_mean_regret}. This immediately yields the convergence rate for top-$M$ mean regret.

\begin{theorem}[Convergence rate for Top-$M$ mean regret]
\label{thm:gen-rate-topM}
Let $\alpha\in(0,1]$, $\delta\in(0,1)$. Let each expert $k$ satisfy
Asm.~\ref{assumption:expert_regret} with $U_k(t,\delta)=O\big(t^{\alpha}c(\delta)\big)$.
Run Alg.~\ref{algorithm:M_LCB} with the confidence bounds of
Def.~\ref{def:ucb_lcb}. Then, with probability at least $1 - \delta$ the mean regret of selected arms $R(T)$ (see Prop.~\ref{prop:m_mean_regret}) is, for all $T\ge 1$,
\begin{equation*}
\mathrm{Reg}_M(T)= O\left(
\sqrt{\frac{KT}{M}\log\frac{KT}{\delta}} +
\Bigl(\frac{K}{M}\Bigr)^{1-\alpha} T^{\alpha} c(\delta)
\right).
\end{equation*}
\end{theorem}
\begin{proof}
The proof is similar to the proof of Theorem~\ref{thm:gen-rate}: one substitutes the width $\Delta_k(n,\delta)$ into the bound of Lemma~\ref{prop:m_mean_regret} and applies the same summation and concavity estimates.
\end{proof}

\subsection{Comparison with existing results}
\label{sec:comparison}

Table~\ref{tab:comparison_existing} lists meta-algorithms used for model selection and budgeted multi-arm training.
Model-selection algorithms such as \cite{pacchiano2020model,cutkosky2021dynamic} achieve order-optimal rates in $T$ and $K$ for single-arm updates but ignore per-round training budgets and subset-level performance; multi-play bandit methods \cite{seldin2014prediction,yun2018multi,kveton2015tight}
handle budgeted updates but do not train arms. \textbf{\algname{M-LCB}} bridges these tasks:
it manages \emph{learnable experts} via \emph{multiple expert updates per round} and extends to the \emph{multi-play} setting with the same order-optimal rate. See Appendix~\ref{sup:applications} for more examples of applicable tasks and base algorithms.

\section{Proofs}
\label{sec:proof}
\setcounter{ineqnumber}{0}

\paragraph{Organization.} This section proves the regret decomposition (Lemma~\ref{lem:regret_bounds}) and the convergence rates (Theorem~\ref{thm:gen-rate}). The remaining statements are proved in the appendix: the lower bound (Theorem~\ref{thm:mt-alpha}), the top-$M$ guarantees (Lemma~\ref{prop:m_mean_regret} and Theorem~\ref{thm:gen-rate-topM}), and all concentration lemmas used below.

\begin{proof}[Proof of Lemma~\ref{lem:regret_bounds}]
\label{proof:regret_bound}
\textbf{Pseudo-regret bound}. Let $k^\star \in \arg\min_{k\in[K]} L_k^\star$ be the index of the best expert.
Algorithm uses the confidence bounds from Prop.~\ref{prop:anytime-ci-clear}, so $\mathcal{E}_\delta$ ~\eqref{eq:conf_intervals} holds with probability at least $1-\delta$; we condition on $\mathcal{E}_{\delta}$ in the sequel. Under event $\mathcal{E}_\delta$, the safe advice at time $t$ is $u^t = \upsilon(\{u_{i_t}^\tau: \tau \in I_{i_t}(t)\})$, and by Asm.~\ref{assumption:wrapper} its conditional expected loss is bounded by $\UCB_{i_t}(t,\delta)$. Therefore,
\begin{equation}\label{eq:best-reg-start}
\widetilde{\mathrm{Reg}}(T) = \mathbb E_{\upsilon}\left[\sum_{t=1}^{T}\bigl[L(u^t) - L^\star\bigr]\right]
\le \mathbb E_{\upsilon}\left[\sum_{t=1}^{T}\bigl[\UCB_{i_t}(t, \delta) - L^\star\bigr]\right],
\end{equation}
where the outer expectation is omitted because all subsequent inequalities hold pointwise. Splitting the sum depending on whether $k^\star$ is trained at $t$ gives
\begin{equation}
\overline{\mathrm{Reg}}(T)
\le \underbrace{\sum_{t:\,k^\star\in S_t} \bigl[\UCB_{i_t}(t, \delta) - L^\star\bigr]}_{\text{A}}
+
\underbrace{\sum_{t:\,k^\star\notin S_t}\bigl[\UCB_{i_t}(t, \delta) - L^\star\bigr]}_{\text{B}}.
\label{eq:best-A+B}
\end{equation}
\emph{Term A.}
Since $i_t=\arg\min_{k\in S_t}\UCB_k(t, \delta)$, for rounds with $k^\star\in S_t$, $\UCB_{i_t}(t, \delta) \le \UCB_{k^\star}(t, \delta),$
and under $\mathcal{E}_{\delta}$, $L^\star\ge \LCB_{k^\star}(t, \delta)$. Therefore,
$\text{A} \le \sum_{t:\,k^\star\in S_t} \bigl[\UCB_{k^\star}(t, \delta)-\LCB_{k^\star}(t, \delta)\bigr] =
\sum_{\tau \in I_{k^\star}(T)} \bigl[\UCB_{k^\star}(\tau, \delta)-\LCB_{k^\star}(\tau, \delta)\bigr].$

\emph{Term B.}
By construction of $i_t$,
$\UCB_{i_t}(t, \delta) \le \frac{1}{M}\sum_{k\in S_t}\UCB_k(t, \delta)$.  Substituting this into term~B in~\eqref{eq:best-A+B}  and applying add–subtract trick with
$\frac{1}{M}\sum_{k \in S_t} \LCB_k(t, \delta)$ for each $t$,
we obtain:
\begin{equation*}
    B \leq \underbrace{\sum_{t: k^\star \notin S_t} \left[ {\frac{1}{M}\sum_{k \in S_t} \LCB_k(t, \delta)} - L^\star \right] }_{\text{C}}  +
    \underbrace{\sum_{t: k^\star \notin S_t} \frac{1}{M}\sum_{k \in S_t} \bigg[\UCB_k(t, \delta) - {\LCB_k(t, \delta)} \bigg] }_{\text{D} }.
\end{equation*}

\textit{Term C.} Since $S_t$ consists of the $M$ smallest LCBs and $k^\star\notin S_t$, every $k\in S_t$ has
$\LCB_k(t,\delta)\le \LCB_{k^\star}(t,\delta)\le L^\star$; hence $\frac{1}{M}\sum_{k\in S_t}\LCB_k(t,\delta)\le L^\star$,
and the term is non-positive.
Therefore,
$B \le \sum_{t:k^\star\notin S_t}
\frac{1}{M}\sum_{k\in S_t}\bigl[\UCB_k(t, \delta)-\LCB_k(t, \delta)\bigr] \le  \frac{1}{M}\sum_{k=1}^{K}\sum_{\tau\in I_k(T)}
\bigl[\UCB_k(\tau, \delta)-\LCB_k(\tau, \delta)\bigr].$
Combining the bounds on A and B with \eqref{eq:best-A+B} proves the pseudo-regret bound.

\textbf{Regret bound.} The proof follows by decomposing regret into a stochastic term and the pseudo-regret:
    $\mathrm{Reg}(T) = \sum_{t=1}^T \big(\ell(u^t,\xi^t) - \mbE_v L(u^t)\big) + \mbE_v \sum_{t=1}^T \big(L(u^t) - L^\star\big).$
Since the data are independent of $u^t$ conditional on the past, the first term is a sum of martingale difference sequence $Z_t =\ell(u^t,\xi^t)-\mathbb E_{\upsilon}[L(u^t)\mid\mathcal F_{t-1}]$ and lies in $[-1,1]$, so it bounded by a concentration inequality (e.g., Lemma~\ref{lemma:azuma_adaptive}). The second term is the pseudo-regret bounded above.
\end{proof}

\begin{proof} [Proof of Theorem~\ref{thm:gen-rate}]
The regret bound in Lemma~\ref{lem:regret_bounds} considers $\UCB_k(t,\delta) - \LCB_k(t,\delta)$,
which by Def.~\ref{def:ucb_lcb} depend only on the number of updates $n^t_k$ and $\delta$.
Let
$\Delta_k(n,\delta)
:= H(n,\delta_{n})
  + \frac{U_k(n,\delta_{\mathrm{arm}})}{n}.$
Substituting $\Delta_k(n^t_k, \delta) = \UCB_k(t, \delta) - \LCB_k(t, \delta)$ into \eqref{eq:bound_bound_bound},
we get ${\Delta(T)} =\sum_{\tau=1}^{n_{k^\star}^T} \Delta_{k^\star}(\tau, \delta) + \frac{1}{M}\sum_{k=1}^K \sum_{\tau =1}^{n^T_k}\Delta_k(\tau, \delta).$
Upper-bounding the slowly growing factors $\log(\delta_n)$ by $\log(\tfrac{KT}{\delta})$ and using the summation bounds
$\sum_{\tau\le n}\tau^{-1/2}=O(\sqrt{n})$ and
$\sum_{\tau\le n}\tau^{\alpha-1}=O(n^\alpha)$,
together with the concavity inequality
$\sum_k (n^t_k)^\alpha \le K^{1-\alpha}(MT)^\alpha$,
we obtain the result.
\end{proof}

\vspace{-1em}
\section{Experiments}
\vspace{-0.5em}
\label{sec:experiments}
We evaluate the proposed \algname{M-LCB} algorithm on real-world and semi-synthetic benchmarks (see Sec.~\ref{sup:real_data} for details).\footnote{The source code is available at \url{https://github.com/ooaarg/mlcb-code}.}Below, we present synthetic results where $K=10$ experts each manage two Bernoulli arms with expected rewards $\mu_{k,1} = \frac{2}{10} \left(1 - \frac{k}{K}\right)$ and $\mu_{k,2} = \frac{3}{10} - \frac{1}{10}\cdot\frac{k}{K}$: as $k$ increases, the best available reward decreases while the internal gap grows, so lower-quality experts converge faster. We assign \algname{EpsilonGreedy} to the first $K/2$ experts and \algname{UCB1} to the rest, creating a ``suboptimal trap'' where a meta-algorithm may be misled by the rapid early convergence of a low-quality expert.
We evaluate \algname{M-LCB} against \algname{LimitedAdvice} and \algname{Smooth Corral} \cite{agarwal2017corralling}, with update limits $M \in \{1, 2, 3, 4\}$, for $T = 2.5 \cdot 10^4$ steps, averaged over 100 independent trials.

Figure \ref{fig:regret_bandits} summarizes the performance. All algorithms achieve sublinear regret, with \algname{M-LCB} and \algname{LimitedAdvice} improving as the budget $M$ grows (Panel a). Increased budgets accelerate convergence to the optimal expert (Panel b) because \algname{M-LCB} adaptively trains the most promising ones (Panel c).

\paragraph{Setup and hyperparameters.}
The confidence bounds of Def.~\ref{def:ucb_lcb} use the functional forms from the finite-sample theory of the base algorithms: for the \algname{UCB1}/\algname{EpsilonGreedy} bandit experts over $m_k$ arms, $U_k(n,\delta)=\sqrt{m_k n \log(n/\delta)}$ with the sampling-concentration correction of Appendix~\ref{sup:bandit_application}; the concentration width is $H(n,\delta_n)=\sqrt{2\log(1/\delta_n)/n}$ with $\delta_{\mathrm{arm}}=\delta/(2K)$ and $\delta_n=\delta/(3.5Kn^2)$. The multiplicative constants were selected once in a separate validation environment (see Appendix~\ref{sup:bandit_experiment}) and fixed afterwards: concentration scale $c=0.5$ and $\delta=0.1$ for \algname{M-LCB}, exploration coefficient $0.5$ and $\varepsilon=1.0$ for the base learners, $\eta=\sqrt{K/T}$ for \algname{Smooth Corral}; \algname{LimitedAdvice} uses the parameters of its theoretical analysis \cite{seldin2014prediction}. The GLM experiment (Appendix~\ref{sup:real_data}) uses the variance-adaptive self-normalized bounds of Appendix~\ref{sec:oco_style_bounds}, with the concentration scale $c=0.03$ selected in a separate GLM validation environment. Each curve is averaged over $100$ independent runs.

\paragraph{Real-data results (condensed).}
On the IWPC Warfarin dosing benchmark ($K=7$ online classifiers; Appendix~\ref{sup:real_data}), \algname{M-LCB} attains the lowest cumulative loss among the compared methods: at $M=2$, loss $6{,}604$ versus $6{,}589$ for the best fixed expert (Figure~\ref{fig:warfarin_theoretical}); with $K\in\{7,10,15\}$ and fixed $M=2$, regret grows with $K$, consistent with Theorem~\ref{thm:gen-rate} (Figure~\ref{fig:warfarin_scaling}). On Cardiotocography ($K=15$), the top-$M$ regret is near zero for $M\ge 4$, matching the best fixed subset as predicted by Theorem~\ref{thm:gen-rate-topM} (Figure~\ref{fig:ctg_topm}).


\begin{figure}[!ht]
    \centering
    \includegraphics[width=0.55\linewidth]{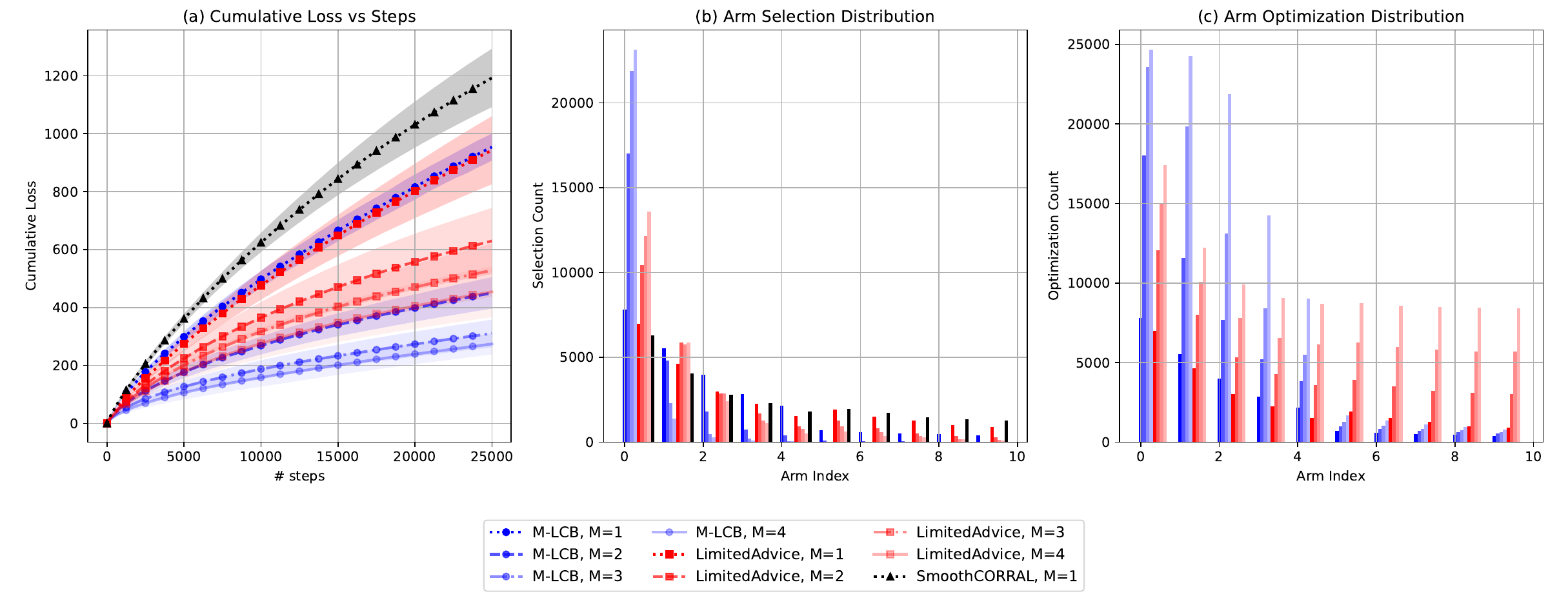}
    \caption{Performance comparison on the MAB model selection.
    (a) Cumulative Loss.
    (b) Final distribution of arm selection.
    (c) Allocation of computational budget across arms.}
    \label{fig:regret_bandits}
\end{figure}

\vspace{-1em}
\section{Limitations and Future Work}
\vspace{-0.5em}
\label{sec:limitations}

Our analysis assumes that the confidence scaling function $c(\delta)$ is known. In the bound of Theorem~\ref{thm:gen-rate}, this function appears only through the expert-learning term $(K/M)^{1-\alpha}T^\alpha c(\delta)$. Consequently, when $\alpha<1/2$ this contribution is asymptotically dominated by the exploration term, whereas for larger $\alpha$ the dependence on $c(\delta)$ is linear. In many standard experts, including OGD, UCB, and Hedge, $c(\delta)$ is polylogarithmic. Practical tuning of the concentration bonus can be mitigated by self-normalized bounds such as those in Section~\ref{sec:oco_style_bounds}, but the dependence on the base experts' confidence calibration is inherent to UCB-type methods.

The case of unknown $c(\delta)$ is considered by Pacchiano et al.~\cite{pacchiano2020model}, where the resulting regret bounds exhibit a weaker dependence on $c(\delta)$, while Dann et al.~\cite{dann2024data} estimate it online, yielding adaptive but less interpretable guarantees. Adapting \algname{M-LCB} to unknown or misspecified expert-rate functions is therefore an important direction for future work.

Another important direction concerns algorithms with additional observations per round,  such as limited-advice and multi-play bandits~\cite{seldin2014prediction,yun2018multi,kveton2015tight}.
Extending their analysis to the setting of \emph{learnable experts} would unify these frameworks.

Another open question is whether meaningful gap-dependent logarithmic guarantees can be obtained. In classical stochastic bandits, the suboptimality gap is fixed, enabling $O(\log T/\Delta)$-type bounds. In our setting the instantaneous regret $L(u_k^t)-L^\star$ changes as experts learn, so the relevant gap is time-dependent and a direct analogue is not immediate.

Finally, a promising extension is the \emph{contextual} regime,
where experts specialize based on observed features or domains,
connecting our framework to contextual bandits~\cite{foster2019model}.



\bibliographystyle{ACM-Reference-Format}
\bibliography{lib_arxiv}


\clearpage

\appendix
\tableofcontents

A reader's guide to this appendix is as follows. Appendix~\ref{sup:usage_examples} instantiates the framework with concrete expert classes and details the domains in which the feedback model of Section~\ref{sec:feedback-domains} holds. Appendix~\ref{sup:numerical} details the numerical experiments summarized in Section~\ref{sec:experiments}. Section~\ref{sec:limitations} discusses limitations of current work. Appendix~\ref{sup:proofs} collects the proofs of the upper-bound analysis deferred from the main text. Appendix~\ref{sec:lower-bounds} proves the lower bound of Theorem~\ref{thm:mt-alpha}. Appendix~\ref{sup:auxiliary} states the standard concentration tools and the two results from \cite{de2009self} used in the proofs.

\section{Usage Examples}
\label{sup:usage_examples}
This section instantiates the framework with concrete expert classes and application domains, expanding Section~\ref{sec:examples} and Section~\ref{sec:feedback-domains} of the main text. We first describe online parametric learners and bandit algorithms as experts, then detail the domains in which the per-round feedback of the protocol is available, and finally summarize representative inner rates in Table~\ref{tab:convergence_examples}.

\subsection{Online parametric learning}
\label{sec:examples_parametric}
In this instantiation, each expert is an online convex optimization algorithm over its own parameter space, and the feedback is the full label of each round.
Let $\xi^t=(x_t,y_t) \overset{\mathrm{i.i.d.}}{\sim} D$, where the distribution $D$ is supported on $\mathbf{E} :=\mathcal X \times \mathcal Y$ and is unkown. The meta-procedure $\mathcal P$ aims to predict $y_t$ observing $x_t$, i.e., it solves a supervised learning task. The decision space is a class of predictors, $\mathbf{U} := \{u: \mathcal X \to \mathcal Y\}$.
$\mathcal P$ manages a pool of experts that are Online Convex Optimization (OCO) algorithms~\cite{hazan2016introduction}.

Each expert $k \in [K]$ is associated with a parameter space $\mathbf W_k \subseteq \mathbb R^{d_k}$, and a prediction function $f_k: \mathbf W_k \times \mathcal X \to \mathcal Y$. For any $w \in \mathbf W_k$, the mapping $f_k(w, \cdot)$ is a predictor in $\mathbf U$, so the decision space of this expert is $\mathbf U_k := \{f_k(w, \cdot):\, w \in \mathbf W_k\} \subseteq \mathbf U.$
The optimal risk for expert $k$ (see~\eqref{def:expeted_smallest_loss}) turns into $L_k^\star = \min_{w \in \mathbf W_k} \mathbb E_{(x, y) \sim D} [\ell(f_k(w, x), y)]$.

When selected for training, expert $k$ predicts $w^t_k \in \mathbf{W}_k$ and receives feedback $\ell(f_k(w^t_k, x_t), y_t)$. For example, if the expert is an Online Gradient Descent (OGD) optimizer, the loss $\ell(f_k(w_k, \cdot), \cdot)$ is assumed to be convex and $G$-Lipschitz on $\mathbf{W}_k$. The expert updates its parameters using the rule $w_k^{\tau+1} = \Pi_{\mathbf{W}_k} [w_k^{\tau} - \eta_{\tau} \nabla_{w} \ell(f_k(w_k^\tau, x_\tau), y_\tau)]$, where $\Pi_{\mathbf{W}_k}$ denotes the Euclidean projection onto $\mathbf{W}_k$, and $\eta_{\tau}$ is learning step size. The OCO regret characterizes the performance of such experts,
\begin{equation}
\label{eq:regret_oco}
\overline{R}_k(t) := \sum_{\tau \in I_k(t)} \ell(u_k^\tau, \xi^\tau) - \min_{u \in \mathbf{U}_k} \sum_{\tau \in I_k(t)} \ell(u, \xi^\tau).
\end{equation}
For OGD, it holds $\overline{R}_k(t) \leq O(GD_k\sqrt{n})$, where
$D_k := \mathrm{diam}(\mathbf{W}_k)$ and $t$ is s.t. $|I_{k}(t)| = n$ \cite{hazan2016introduction}.

Lemma~\ref{lem:prefix_to_expected} relates this regret to the expected regret (see Asm.~\ref{assumption:expert_regret}) via a concentration bound. Consequently, OGD experts satisfy the assumption with
$U_k(n, \delta) = O\left(GD_k\sqrt{n} + \sqrt{n \log(n/\delta)}\right)$.

\subsection{Multi-armed bandit algorithms as experts.}
\label{sup:bandit_application}
 Motivated by AutoML and distributed routing, we consider a meta-procedure $\mathcal{P}$ that selects the most efficient among $K$ self-optimizing experts. As in Bandit Model Selection \cite{pacchiano2020model}, this is inherently challenging due to the experts' non-stationary learning dynamics and the high sample cost required for convergence.

Each expert $k$ is a stochastic bandit managing a finite action space $\mathcal{A}_k := \{a_k^1, \dots, a_k^{m_k}\}$. An expert's decision at time $t$ is a probability distribution $u_k^t \in \Delta_k$ over its actions, where $\mathbf{U}_k = \Delta_k$ is the $m_k$-dimensional simplex. The global decision space $\mathbf{U}$ is the union of simplices.

The environment is characterized by a latent random variable $\xi^t \sim D$ that represents the state of the system at time $t$ (e.g., a vector of current latencies or resource costs for all available methods). Although $\xi^t$ need not be observed in full by the meta-procedure, it determines the loss of each action $a \in \bigcup_k \mathcal{A}_k$ via the mapping $\tilde{\ell}(a, \xi)$.
The expert loss for a decision $u_k \in \mathbf{U}_k$ is $\ell(u_k, \xi) := \mathbb{E}_{a \sim u_k} [\tilde{\ell}(a, \xi)]$. The optimal risk for expert $k$ is $L_k^\star = \min_{a \in \mathcal{A}_k} \mathbb{E}_{\xi \sim D} [\tilde{\ell}(a, \xi)]$.

When expert $k\in S_t$ produces $u^{t}_k$, it samples a discrete action $a_k^{t} \in \mathcal{A}_k$ from $u_k^{t}$ and receives $(a^{t}_k, \tilde{\ell}(a^{t}_k, \xi^{t}))$ as bandit feedback. Thus, in this instantiation the limited-advice assumption means that the environment provides the sampled-action loss for each trained expert $k\in S_t$, including experts that are not deployed as the global advisor. If several trained experts sample the same action, the corresponding loss need only be evaluated once and can be reused for all of them. When off-policy updates are available, one may also update base algorithms from additional samples, but our analysis uses only the on-policy observations above.

In this setting, the expert's performance is measured by its pseudo-regret $R_k(n)$, which evaluates the cumulative expected loss of the distributions $u_k^\tau$ relative to the best fixed action, $R_k(t) = \sum_{\tau \in I_k(t)} \mathbb{E}_{a \sim u_k^\tau} [\tilde{\ell}(a, \xi^\tau)] - n_k L_k^\star.$

Standard algorithms such as \algname{UCB} satisfy Asm.~\ref{assumption:expert_regret} with $U_k(n, \delta) = O(\sqrt{m_k n \log(n/\delta)})$. The theoretical running loss $L_k(t)$ in~\eqref{def:experts_running_loss} is defined through the conditional loss $\ell(u_k^t,\xi^t)$, whereas the bandit expert observes the sampled loss $\tilde{\ell}(a_k^t,\xi^t)$. In implementation one therefore uses $\tilde L_k(t)=\frac{1}{n_k^t}\sum_{\tau\in I_k(t)}\tilde{\ell}(a_k^\tau,\xi^\tau),$
and adds the corresponding sampling-concentration term to $U_k$. The concentration of $\sum_{\tau\in I_k(t)} \tilde{\ell}(a_k^\tau, \xi^\tau)$ around $\sum_{\tau\in I_k(t)} \ell(u_k^\tau, \xi^\tau)$ ensures that anytime high-probability bounds on pseudo-regret hold. The proof is similar to Lemma~\ref{lem:prefix_to_expected}.

\subsection{Feedback model: detailed domains}
\label{sup:feedback_domains}
We detail the domains sketched in Section~\ref{sec:feedback-domains}, specifying for each of them the expert classes from the previous subsections and their anytime rates $U_k(n,\delta)$ (see also Table~\ref{tab:convergence_examples}).

\emph{Full-feedback online supervised learning.} The outcome is a pair $\xi^t=(x_t,y_t)$, and once the label $y_t$ is revealed, the loss of every model is computable, whether or not its prediction was acted upon. This is the standard full-information online learning protocol \cite{hazan2016introduction}, and it covers, e.g., clinical decision support, where the physician follows one recommendation but all candidate models are scored retrospectively against the realized outcome; our real-data experiments (Appendix~\ref{sup:real_data}) have exactly this form. Experts: OCO algorithms with $U_k(n,\delta)=O(G_kD_k\sqrt{n}+\sqrt{n\log(n/\delta)})$.

\emph{Shadow deployment and A/B testing.} Production recommender and ranking systems routinely score every request with all candidate models in shadow mode; only one model's output is served, but every model's loss against the realized user feedback is logged. The budget $M$ is the training (GPU) budget deciding which shadowed models are updated. Experts: OCO algorithms as above, or Hedge aggregators with $U_k(n,\delta)=\sqrt{(n/2)\log(N_k/\delta)}$.

\emph{Simulator and digital-twin training.} When a faithful simulator exists (robotics, network control, pharmacokinetic simulators for precision dosing), one policy is deployed in production while $M$ candidate policies are evaluated and updated in simulation on the same realized context; the simulator supplies exactly the loss of decisions not played in reality. Experts: bandit algorithms with $U_k(n,\delta)=O(\sqrt{m_k n\log(n/\delta)})$ or bandit convex optimization with $O(n^{3/4})$--$O(n^{5/6})$ rates \cite{flaxman2004online}.

\emph{Network routing with probing.} Production traffic follows the deployed routing policy, while up to $M$ alternative paths are actively probed per round \cite{awerbuch2008online}; link-level telemetry then reveals the loss of each probed path directly. Experts: OCO over path spaces with $O(\sqrt{n})$-type rates, or one-point bandit feedback with $O(n^{3/4})$ \cite{flaxman2004online}.

\emph{Off-policy estimation.} Where none of the above applies and only deployed-action feedback exists, losses of undeployed experts can still be estimated from logged bandit feedback via importance-sampling or doubly robust estimators \cite{dudik2011doubly,swaminathan2015counterfactual,saito2021open}; such estimates can feed the empirical loss in our bounds at the price of inflating $U_k$ by the importance-weight range. We leave a formal treatment of this case to future work (see Section~\ref{sec:limitations}).

\subsection{Specific examples}
\label{sup:applications}

\begin{remark}[On constants]
\label{rem:beta_dependence}
If the constants $\beta_k$ in the inner bounds, $U_k(t,\delta)=O(\beta_k\,t^{\alpha}c(\delta))$ are significant
(e.g., depend on the dimension or Lipschitz constant of the task),
then the regret bound refines to
\begin{equation}
O\left(
\sqrt{\tfrac{KT}{M}\log(\tfrac{KT}{\delta})}
+ T^{\alpha} c(\delta) \left[ \beta_{k^\star} +
\Bigl(\tfrac{1}{M}\sum_{k=1}^K \beta_k^{\frac{1}{1-\alpha}}\Bigr)^{1-\alpha} \right] \right),
\end{equation}
providing a more precise characterization of the dependence on individual expert complexities.
\end{remark}

Table~\ref{tab:convergence_examples} summarizes several representative base algorithms,  their inner convergence rates, and the resulting global regret when combined with \algname{M-LCB}.
Alongside standard convex and exponential-weighted learners,  we include ``hard'' stochastic problems with heavy-tailed rewards, which exhibit slower convergence characterized by larger $\alpha$.

Different experts may operate over distinct action spaces.
For instance, in bandit-based experts, each learner may control its own set of arms,   while in parametric or linear models, the dimensionality of the feature space may vary.  Such heterogeneity is reflected in parameters like $N_k$ or $d_k$ in Table~\ref{tab:convergence_examples},   and is naturally handled within the \algname{M-LCB} framework.

\section{Numerical Experiments}
\label{sup:numerical}

This section details the experiments summarized in Section~\ref{sec:experiments}: the GLM model-selection benchmark, the online selection of bandit algorithms, and the two real-data benchmarks. The experiments do not require a GPU; they were conducted on a MacBook M1 Pro with 16 GB RAM.

\subsection{Model Selection among Generalized Linear Models}
We evaluate the performance of our proposed algorithm, \algname{M-LCB}, on synthetic problems designed to test its ability to manage adaptive arms under a computational budget.
We compare against two baselines:
\EDRB \cite{dann2024data}, an algorithm with guarantees for learnable arms,
and \algname{LimitedAdvice} \cite{seldin2014prediction}, an expert algorithm capable of handling multiple arm updates per round.
In all experiments we consider update limits $M \in \{1,2,3\}$, average results over 30 independent runs, and display $\pm 0.5$ standard deviations as shaded regions.

We consider a model selection problem with $K=10$ arms.
Each arm $k$ represents a generalized linear model (GLM) with a distinct, fixed link function $f_k:\mathbb R\to\mathbb R$.
At each round $t$, a feature vector $x_t \in \mathbb R^d$ is drawn uniformly from the unit sphere $\mathbb S^{d-1}$,
and the label is generated by the optimal arm $k^\star=9$ as
\begin{equation}
r_t = f_{k^\star}(w^\top x_t),
\end{equation}
As illustrated in Figure~\ref{fig::glm_functions}, the link functions $f_7, f_8,$ and $f_9$ are highly similar in regions where the data is dense,
presenting a nontrivial exploration challenge.

\begin{figure}
    \centering
    \includegraphics[width=0.5\linewidth]{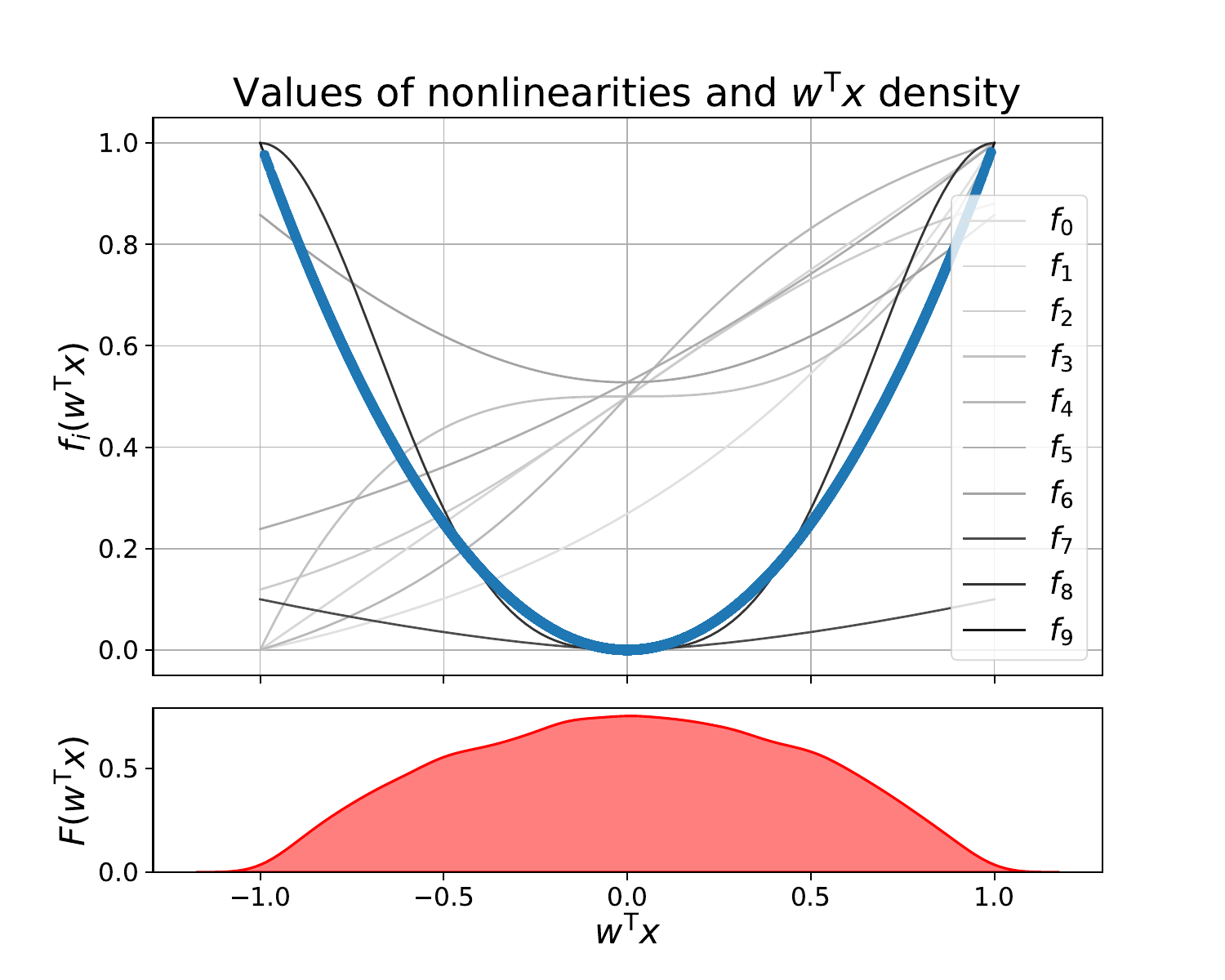}
    \caption{Nonlinear link functions associated with the arms (top) and the density of generated data points (bottom).
    One can see that the last three functions are highly similar where the data is concentrated, making it hard to distinguish the optimal arm.}
    \label{fig::glm_functions}
\end{figure}

\paragraph{Results.}
Figure~\ref{fig:exp2} summarizes the results.
Panel (a) shows that \algname{M-LCB} achieves sublinear regret and is competitive with both baselines.
Panel (b) reports the final arm selection distribution: \algname{M-LCB} successfully identifies the optimal arm ($k=9$).
Panel (c) presents the distribution of the computational budget across arms.
\algname{M-LCB} allocates updates primarily to top-performing arms, while \algname{LimitedAdvice} spreads updates more evenly, leading to less efficient use of the training budget.

\begin{figure}[!ht]
    \centering
    \includegraphics[width=0.6\linewidth]{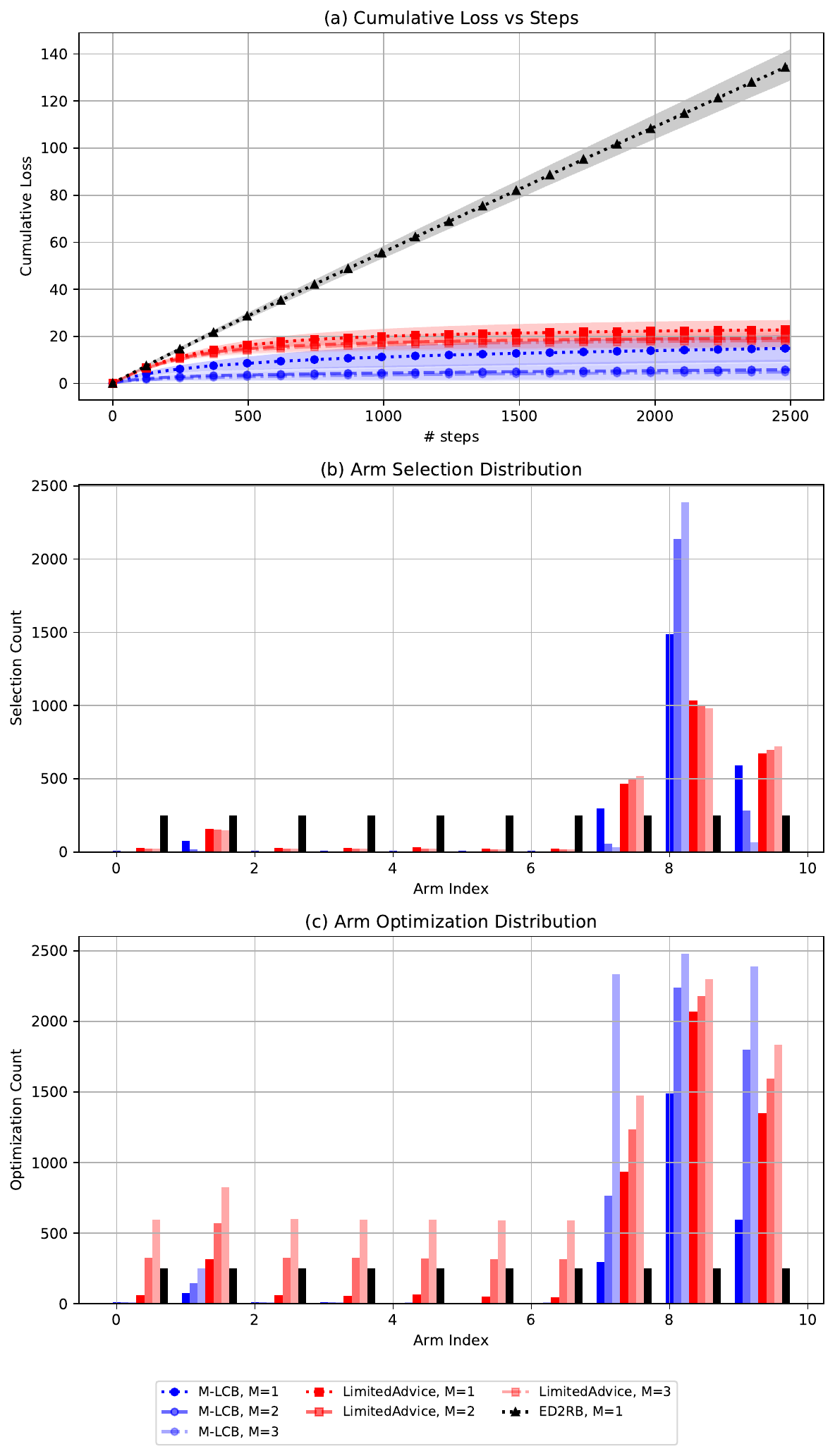}
    \caption{Performance comparison on the GLM model selection problem.
    (a) Cumulative regret.
    (b) Final distribution of arm selection.
    (c) Allocation of computational budget across arms.}
    \label{fig:exp2}
\end{figure}

\subsubsection*{Hyperparameters}
For \EDRB, the exploration parameter was tuned, with $c=0.1$ giving the best results.
For \algname{M-LCB}, the multiplicative constant of the concentration terms was selected once in a separate validation environment and fixed afterwards: the validation environment consists of $K=6$ GLM experts built from the first six link functions (the label generated by the last of them) with update budget $M=3$ and horizon $T=2.5\cdot 10^3$; sweeping the concentration scale $c$ over $\{0.003, 0.01, 0.03, 0.1, 0.3, 1.0\}$ ($24$ trials per value) yields a nearly flat curve with the minimum at $c=0.03$, which we fix for this experiment together with $\delta=0.1$.
Parameters for \algname{LimitedAdvice} were set according to its theoretical analysis \cite{seldin2014prediction}. Analysis shows that regret is not highly sensitive to these parameters.

\subsection{Online Selection of Bandit Algorithms}
\label{sup:bandit_experiment}

In this experiment, we instantiate our framework to solve a model selection problem where each expert $k \in [K]$ is itself an independent stochastic Multi-Armed Bandit (MAB) algorithm.

\paragraph{Setup.}
We use $K=10$ experts, each managing two Bernoulli arms with expected rewards $\mu_{k,1} = \frac{2}{10}\left(1 - \frac{k}{K}\right)$ and $\mu_{k,2} = \frac{3}{10} - \frac{1}{10}\cdot\frac{k}{K}$. The first $K/2$ experts are \algname{EpsilonGreedy} with $\varepsilon=1.0$, the remaining ones are \algname{UCB1}; all base learners use exploration coefficient $0.5$. The meta-algorithms are \algname{M-LCB} and \algname{LimitedAdvice} with $M\in\{1,2,3,4\}$ and \algname{SmoothCORRAL} with $\eta=\sqrt{K/T}$; the horizon is $T=2.5\cdot 10^4$, and the results are averaged over $100$ independent trials. \algname{M-LCB} runs with concentration scale $c=0.5$ and confidence level $\delta=0.1$.

\paragraph{Selection of the concentration constants.}
The multiplicative constants of the concentration terms were selected once in a separate validation environment and then fixed for all experiments. The validation environment consists of $K=6$ experts, all of them \algname{EpsilonGreedy} bandits over the same family of reward grids, with update budget $M=3$ and horizon $T=2.5\cdot 10^4$. Sweeping the concentration scale $c$ over the grid $\{0.05, 0.1, 0.2, 0.3, 0.5, 0.75, 1.0\}$ ($150$ trials per value) shows that the final cumulative loss is nearly flat: all grid points lie within two standard errors of the best one. This indicates low sensitivity of the regret to moderate over- or under-estimation of the constants, and we fix $c=0.5$ for all reported experiments.

\begin{figure}[!ht]
    \centering
    \includegraphics[width=0.8\linewidth]{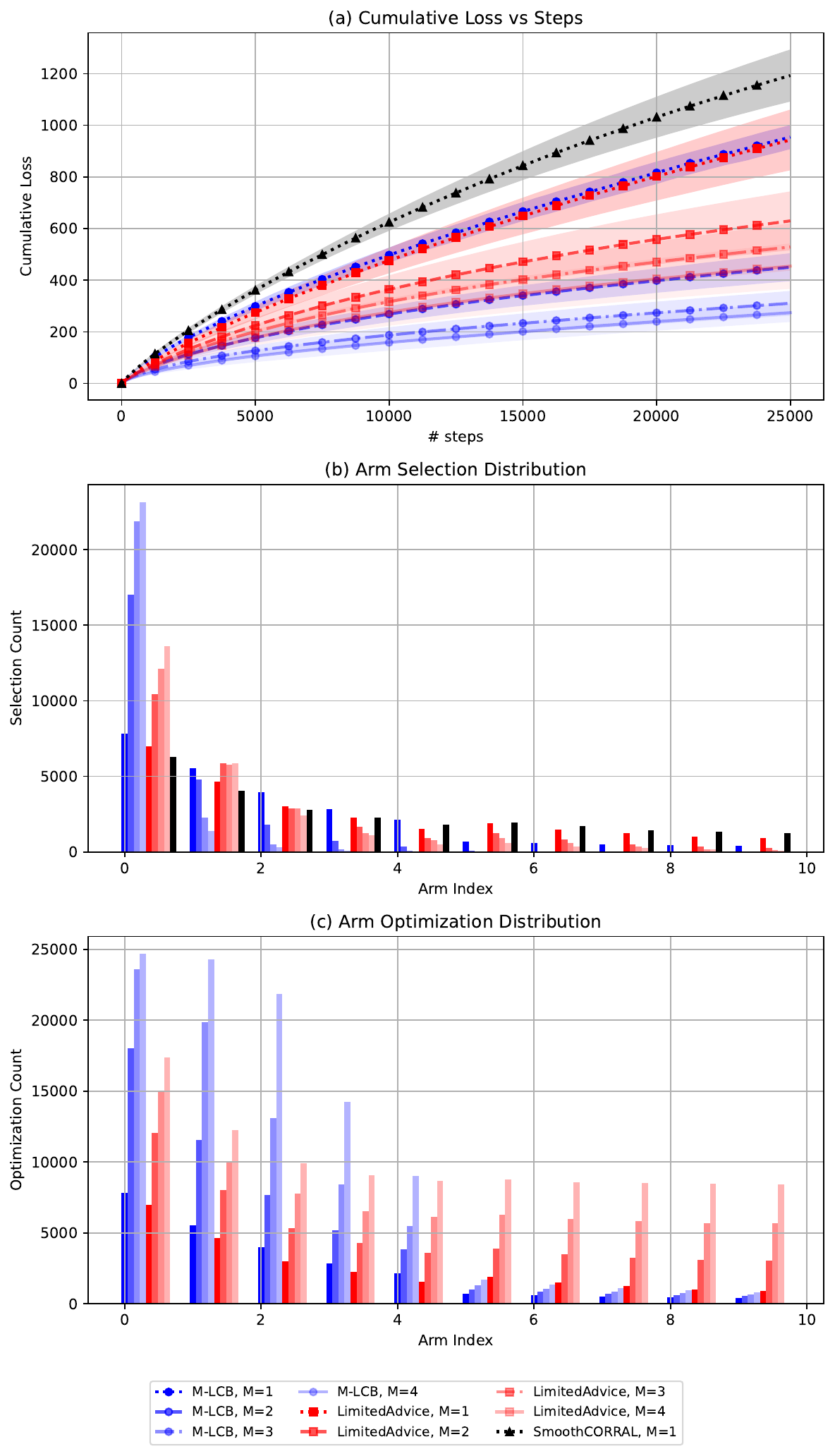}
    \caption{Performance comparison on the MAB model selection.
    (a) Cumulative regret.
    (b) Final distribution of arm selection.
    (c) Allocation of computational budget across arms.}
    \label{fig:regret_bandits_appendix}
\end{figure}

\begin{table*}[h!]
\centering
\caption{%
Examples of inner-arm convergence rates $U_k(n,\delta)$ and resulting global regret (up to logarithmic factors and an additive term $O(\sqrt{(K/M)T})$).
For each expert $k$, the inner algorithm satisfies $U_k(t,\delta)=O(\beta_k\,t^{\alpha}c(\delta))$,
and the corresponding global regret scales as
$O\!\big(T^{\alpha}c(\delta)\,\|{\boldsymbol{\beta}}\|_{M, 1-\alpha}\big)$,
where
$\|{\boldsymbol{\beta}}\|_{M,\gamma}
=\Bigl(\tfrac{1}{M}\sum_{k=1}^{K}\beta_k^{\frac{1}{\gamma}}\Bigr)^{\gamma}$.
Parameter conventions:
$K$ — \# experts; $M$ — per-round training budget; $T$ — horizon;
$N_k$ — \# base arms/actions; $d_k$ — feature dimension;
$\varepsilon$ — heavy-tail moment exponent ($\mathbb{E}|X|^{1+\varepsilon}\!\le\sigma^{1+\varepsilon}$);
$L, R, C, G$ — Lipschitz, diameter, range, and gradient constants.
}
\label{tab:convergence_examples}
\begin{tabular}{p{5.0cm}p{3.7cm}p{4.0cm}}
\toprule
\textbf{Inner algorithm / problem}
& \textbf{Inner rate $U_k(n,\delta)$}
& \textbf{Global regret (up to logs)} \\
\midrule
OGD / OMD (convex Lipschitz)
& $O(G_kR_k\sqrt{n})$, $\alpha=\tfrac{1}{2}$,
& $O\!\big(T^{1/2}\,\|GD\|_{M,1/2}\big)$ \\

\midrule
Bandit Convex Optimization (bounded $f_t$)~\cite{flaxman2004online}
& $O(C_kd_kn^{5/6})$, $\alpha=\tfrac{5}{6}$
& $O\!\big(T^{5/6}\,\|Cd\|_{M,1/6}\big)$ \\

\midrule
Bandit Convex Optimization ($L$–Lipschitz $f_t$)~\cite{flaxman2004online}
& $O(\sqrt{C_kL_kR_k}d_kn^{3/4})$, $\alpha=\tfrac{3}{4}$
& $O\!\big(T^{3/4}\,\|\sqrt{CLR}d\|_{M,1/4}\big)$ \\

\midrule
Heavy–tailed stochastic bandits~\cite{bubeck2013bandits}
& $\tilde O(n^{\alpha}N_k^{1-\alpha})$, $\alpha=\tfrac{1}{1+\varepsilon}$
& $O\!\big(T^{\alpha}\, [\frac{1}{M}\sum_{k=1}N_k]^{1 - \alpha})$ \\

\midrule
Heavy–tailed stochastic bandits (Symmetric noise)~\cite{dorn2024fast}
& $O(\sqrt{N_k n} {\log n}^{\frac{3}{2}})$, $\alpha=\tfrac{1}{2}$
& $O\!\big(T^{1/2}\,[\frac{1}{M}\sum_{k=1}N_k]^\frac{1}{2}\big)$ \\

\midrule
Heavy–tailed linear bandits~\cite{tajdini2025improved}
& $O(d_k^{\frac{3}{2} - \alpha} n^{\alpha})$, $\alpha=\tfrac{1}{1 + \varepsilon}$
& $O\big(T^{\alpha}\,\|d^{\frac{3}{2} - \alpha}\|_{M,1 - \alpha}\big)$ \\

\midrule
Hedge / Exponential Weights (over $N_k$ experts)
& $O(\sqrt{n\log N_k})$, $\alpha=\tfrac{1}{2}$,
& $O\!\big(T^{1/2}\,\|\sqrt{\log N}\|_{M,1/2}\big)$ \\
\bottomrule
\end{tabular}
\end{table*}

\subsection{Real Data Experiments}
\label{sup:real_data}

We further evaluate \algname{M-LCB} on two real multiclass clinical prediction benchmarks that stress different aspects of the limited-feedback setting. The first benchmark tests long-horizon regret behavior through a bootstrapped stream from a moderate-size pharmacogenomics dataset, while the second uses a fully real, non-bootstrapped stream to evaluate performance under repeated random orderings of a smaller medical monitoring dataset.

\paragraph{Warfarin dosing.}
We first consider the IWPC Warfarin dosing dataset \cite{international2009estimation, bastani2020mostly}, which contains 5,528 patients, 93 features, and a three-class dosing label. Because the dataset is smaller than the horizon required to expose asymptotic regret trends, we generate a stream of $T=20{,}000$ rounds by bootstrap resampling. The expert pool consists of $K=7$ heterogeneous online softmax and perceptron classifiers constructed using different model classes and feature subsets. We compare methods under update budgets $M\in\{1,2,4\}$.

The results support the predicted dependence on the advice budget. As shown in Figure~\ref{fig:warfarin_theoretical}, \algname{M-LCB} obtains the lowest cumulative loss among the compared online methods. For example, when $M=2$, \algname{M-LCB} achieves cumulative loss $6{,}604$, close to the best-fixed-expert reference of $6{,}589$. This small gap indicates that, despite observing and updating only a limited number of experts per round, the algorithm rapidly concentrates updates on competitive experts and tracks the best fixed comparator. We also vary the number of experts with $K\in\{7,10,15\}$ while fixing $M=2$. Figure~\ref{fig:warfarin_scaling} shows that regret increases with $K$, consistent with the $\mathcal{O}(\sqrt{KT/M})$ dependence in Theorem~\ref{thm:gen-rate-topM}.

\begin{figure}[t]
    \centering
    \includegraphics[width=0.95\linewidth]{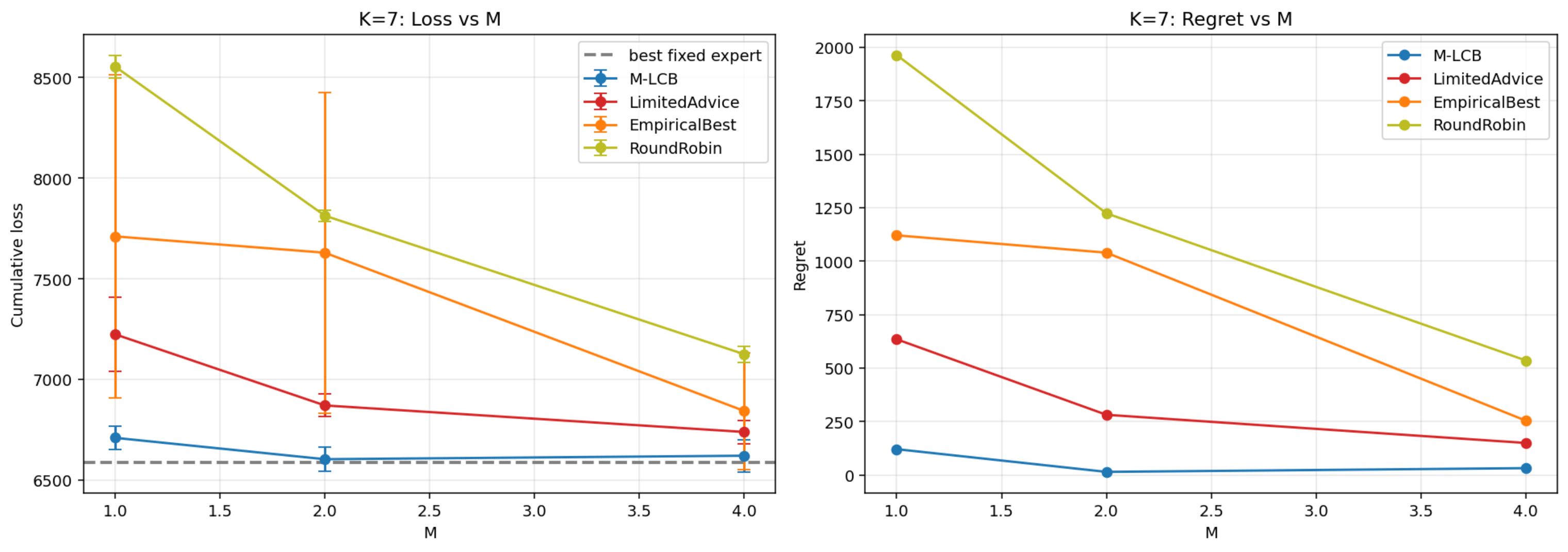}
    \caption{Warfarin dosing results under limited advice budgets. \algname{M-LCB} achieves the lowest cumulative loss among the compared methods and, at $M=2$, attains cumulative loss $6{,}604$, close to the best-fixed-expert reference of $6{,}589$.}
    \label{fig:warfarin_theoretical}
\end{figure}

\begin{figure}[t]
    \centering
    \includegraphics[width=0.95\linewidth]{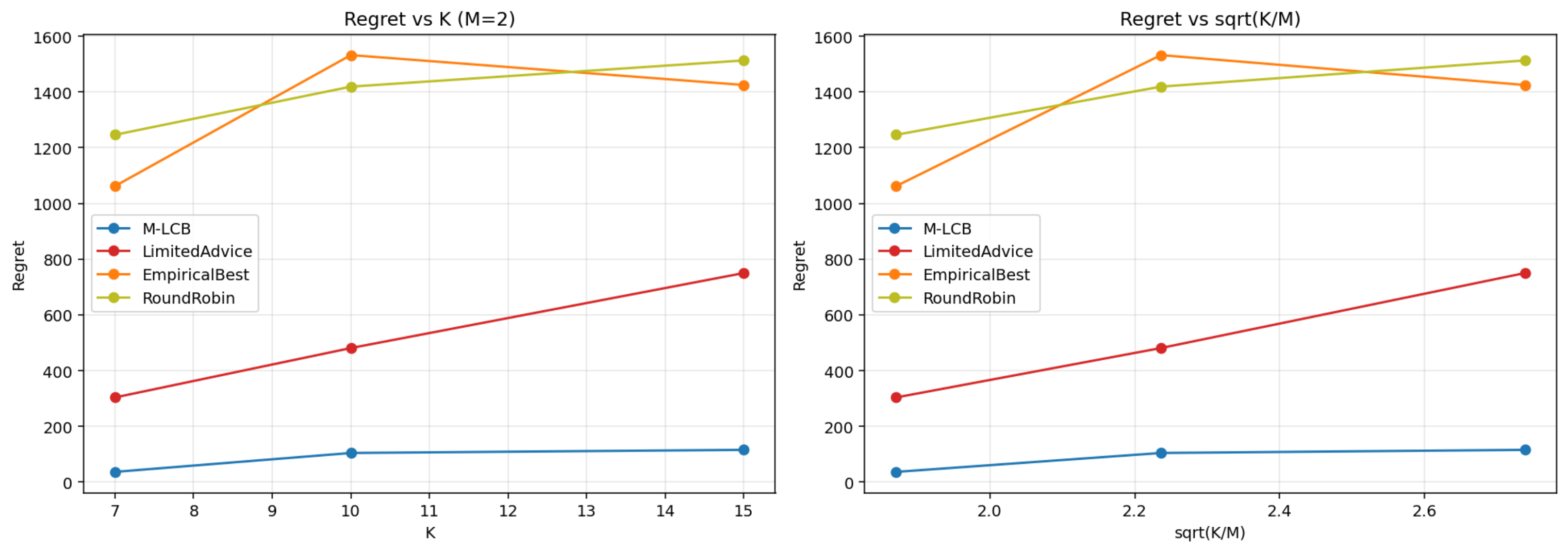}
    \caption{Scaling behavior on the Warfarin dosing benchmark with $K\in\{7,10,15\}$ and fixed budget $M=2$. Regret grows with the number of experts, matching the qualitative prediction of the $\mathcal{O}(\sqrt{KT/M})$ regret dependence.}
    \label{fig:warfarin_scaling}
\end{figure}

\begin{figure}[h!t]
    \centering
    \includegraphics[width=0.8\linewidth]{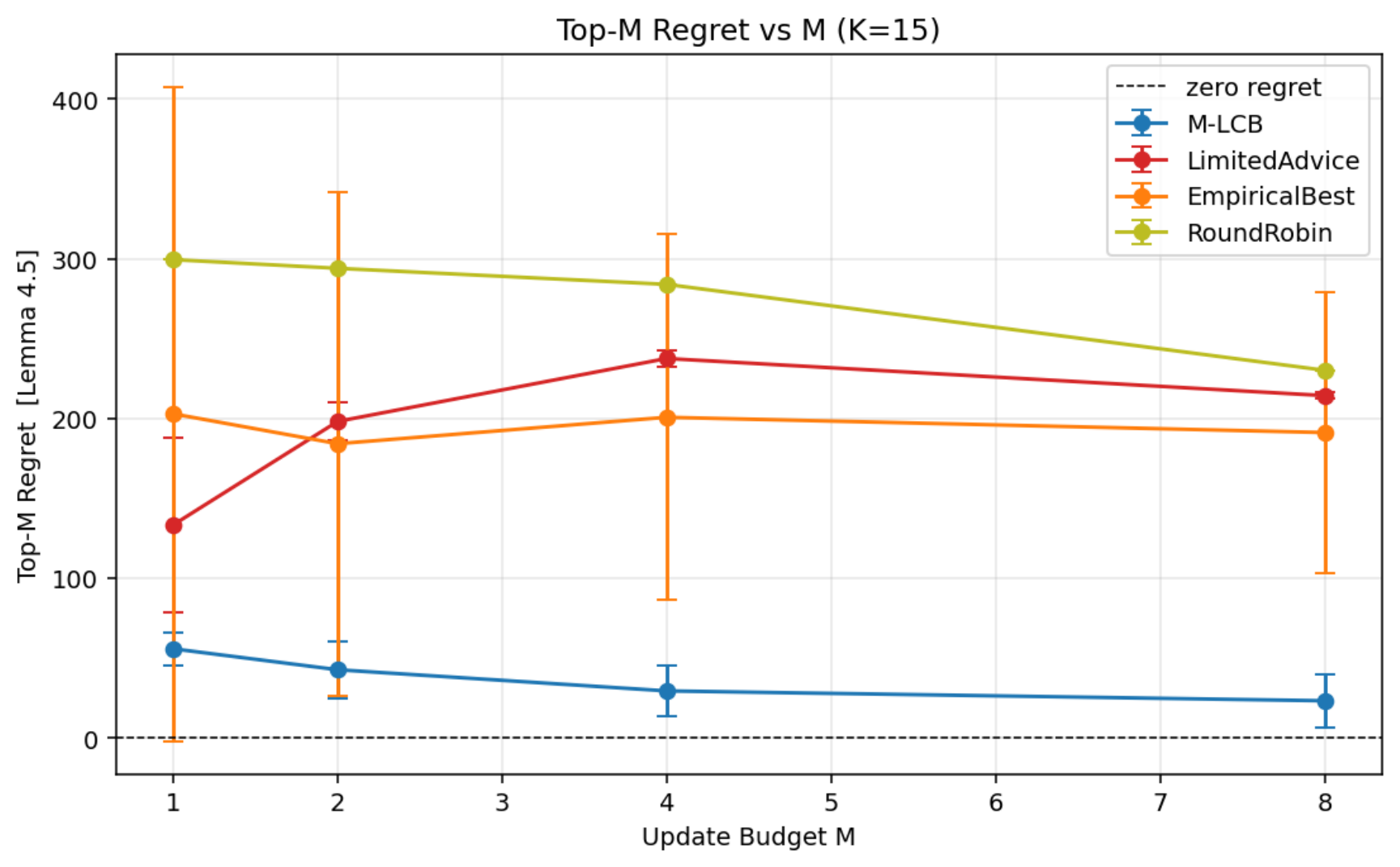}
    \caption{Top-$M$ regret on the Cardiotocography benchmark with $K=15$ experts and budgets $M\in\{1,2,4,8\}$ over 30 random seeds. For $M\ge4$, \algname{M-LCB} obtains near-zero top-$M$ regret, matching the best fixed subset of $M$ experts.}
    \label{fig:ctg_topm}
\end{figure}

\paragraph{Cardiotocography.}
We next evaluate on the Cardiotocography dataset \cite{cardiotocography_193, bastani2020mostly}, which contains 2,126 fetal monitoring traces, 35 features, and a three-class fetal-state label. Unlike the Warfarin experiment, this benchmark uses a fully real, non-bootstrapped stream. We use an expert pool with $K=15$ experts and vary the advice budget over $M\in\{1,2,4,8\}$ across 20 random seeds.

Experimental results on the Cardiotocography dataset reveal that the quality of the learned mixture improves monotonically with the update budget.
Under the configuration $M=4$, \algname{M-LCB} attains $95.0\%$ mean accuracy and $105.6$ cumulative loss; for reference, the optimal fixed expert achieves a loss of $86.0$.
Figure~\ref{fig:ctg_topm} depicts the top-$M$ regret, where \algname{M-LCB} consistently surpasses \algname{LimitedAdvice}, \algname{RoundRobin}, and \algname{SmoothCORRAL}.
This performance differential underscores the efficacy of adaptive budget allocation over static or stochastic update policies.
Notably, the observation that \algname{M-LCB} operates on par with the best fixed subset of $M$ experts offers strong empirical support for the theoretical subset-level guarantee presented in Theorem~\ref{thm:gen-rate-topM}.

\section{Proofs}
\label{sup:proofs}

This section collects the proofs of the upper-bound analysis deferred from the main text: the conversion of OCO-style regret to the anytime expert regret of Assumption~\ref{assumption:expert_regret} (Lemma~\ref{lem:prefix_to_expected}), the anytime confidence bounds behind Definition~\ref{def:ucb_lcb} (Lemma~\ref{prop:anytime-ci-clear}), the variance-adaptive self-normalized bounds used in the GLM experiment, and the proof of the top-$M$ guarantee (Lemma~\ref{prop:m_mean_regret}).

\begin{lemma}[Relating OCO Regret to Expert Regret]
\label{lem:prefix_to_expected}
Consider an expert $k$ as described in Section~\ref{sec:examples_parametric}. Suppose that the expert provides an anytime high-probability bound $\overline{U}_k(t, \delta)$ on the OCO-style regret $\overline{R}_k(t)$ defined in \eqref{eq:regret_oco}. That is, for any fixed $\delta \in (0,1)$:
$
    \mathbb{P}\left( \forall t \ge 1, \quad \overline{R}_k(t) \le \overline{U}_k(t, \delta) \right) \ge 1 - \delta.
$
Then, the expert satisfies Asm.~\ref{assumption:expert_regret} with respect to the regret $R_k(t)$ with the bound:
\begin{equation*}
U_k(t, \delta) = \overline{U}_k(t, \delta/2) + t \cdot G(t, \delta/(3.5t^2)),
\end{equation*}
where $G(t, \delta)$ is the concentration term:
\begin{equation*}
G(t, \delta) = \sqrt{\frac{2 \log(1/\delta)}{t}} + \frac{2 \log(1/\delta)}{3t}.
\end{equation*}
\end{lemma}

\begin{proof}
For any time $t$, let $n = |I_k(t)|$ denote the number of rounds where expert $k$ was updated. To simplify notation, we relabel the learning rounds of expert $k$ as $\tau = 1, \dots, n$ and denote the corresponding decisions and losses as $u_k^\tau$ and $\xi^\tau$, respectively.
We fix $\delta \in (0, 1)$ and define the concentration term $G(n, \delta') = \sqrt{\frac{2 \log(1/\delta')}{n}} + \frac{2 \log(1/\delta')}{3n}$.

We define two high-probability events. First, let $\mathcal{E}_1$ be the event where the anytime OCO-style regret bound holds for expert $k$. By the assumption in the Lemma, with probability at least $1 - \delta/2$:
\begin{equation}
\label{eq:emp_reg_proof}
    \sum_{\tau=1}^n \ell(u_k^\tau, \xi^\tau) - \min_{u \in \mathbf{U}_k} \sum_{\tau=1}^n \ell(u, \xi^\tau) \le \overline{U}_k(n, \delta/2) \quad \forall n \ge 1.
\end{equation}

Second, let $u_k^\star \in \arg \min_{u \in \mathbf{U}_k} L(u)$ be the optimal fixed decision in expectation, s.t. $L(u_k^\star) = L_k^\star$. Applying Asm.~\ref{assumption:stochastic_losses} and Lemma~\ref{lem:loss-eval} to the sequence of i.i.d. losses $\{\ell(u_k^\star, \xi^\tau)\}_{\tau=1}^n$, we have that for a fixed $n$, with probability at least $1 - \delta/(3.5n^2)$:
\begin{equation*}
    \frac{1}{n} \sum_{\tau=1}^n \ell(u_k^\star, \xi^\tau) \le L_k^\star + G\bigl(n, \delta / (3.5n^2)\bigr).
\end{equation*}
Taking a union bound over all $n \ge 1$, and noting that $\sum_{n=1}^\infty \frac{1}{3.5n^2} = \frac{1}{3.5} \frac{\pi^2}{6} < 0.5$, we define $\mathcal{E}_2$ as the event where this inequality holds for all $n$ simultaneously. Thus, $\mathbb{P}(\mathcal{E}_2) \ge 1 - \delta/2$, and under $\mathcal{E}_2$ for all $n \ge 1$ holds
\begin{equation}
\label{eq:comp_conc_proof}
    \sum_{\tau=1}^n \ell(u_k^\star, \xi^\tau) \le n L_k^\star + n G\bigl(n, \delta / (3.5n^2)\bigr).
\end{equation}

By the union bound, $\mathbb{P}(\mathcal{E}_1 \cap \mathcal{E}_2) \ge 1 - \delta$. Under this intersection, we utilize the optimality of the empirical minimizer to observe that $\min_{u \in \mathbf{U}_k} \sum_{\tau=1}^n \ell(u, \xi^\tau) \le \sum_{\tau=1}^n \ell(u_k^\star, \xi^\tau)$. Combining \eqref{eq:emp_reg_proof} and \eqref{eq:comp_conc_proof}, we obtain for all $n \ge 1$:
\begin{align*}
    \sum_{\tau=1}^n \ell(u_k^\tau, \xi^\tau) &\le \min_{u \in \mathbf{U}_k} \sum_{\tau=1}^n \ell(u, \xi^\tau) + \overline{U}_k(n, \delta/2) \\
    &\le \sum_{\tau=1}^n \ell(u_k^\star, \xi^\tau) + \overline{U}_k(n, \delta/2) \\
    &\le n L_k^\star + \overline{U}_k(n, \delta/2) + n G\bigl(n, \delta / (3.5n^2)\bigr).
\end{align*}

Finally, rearranging the last inequality, we get inequality in as in Assumption \ref{assumption:expert_regret}. Which concludes the proof.
\end{proof}

\begin{remark}
\label{rem:tighter-prefix}
    In the proof of Lemma \ref{lem:prefix_to_expected} when applying Lemma~\ref{lem:loss-eval} for bounded losses one may use $L_k^* \leq Z_k(t, \delta) := \min(1, \UCB_k(t, \delta))$ to get a tighter bound $G_k(t, \delta) \coloneqq \sqrt{\frac{2 Z_k(t, \delta) \log(1/\delta)}{3t}} + \frac{2 \log(1/\delta)}{t}$.
\end{remark}

\begin{lemma}[Anytime confidence bounds]
\label{prop:anytime-ci-clear}
Let
\begin{equation}
\label{eq:conf_intervals}
    \mathcal{E}_\delta := \left\{\forall k \in [K],~ \forall t \ge 1 \quad
    \begin{array}{l}
        \LCB_k(t,\delta) \le L_k^\star \\
        L_k^\star \le \UCB_k(t,\delta) \\
        \mathbb E[L(\upsilon(H_k^t))\mid H_k^t] \le \UCB_k(t,\delta)
    \end{array}
    \right\},
\end{equation}
where $H_k^t:=\{u_k^\tau:\tau\in I_k(t)\}$ denotes the history of decisions used by the wrapper for expert $k$ at time $t$.

Under Assumptions~\ref{assumption:stochastic_losses}-\ref{assumption:wrapper} for any $\delta \in (0, 1)$ it holds  $\mathbb{P}(\mathcal{E}_\delta) \ge 1 - \delta$.
\end{lemma}

\begin{proof}[Proof of Lemma~\ref{prop:anytime-ci-clear}]

For any time $t$, let $n = |I_k(t)|$ denote the number of rounds where expert $k$ was updated. To simplify notation, we relabel the learning rounds of expert $k$ as $\tau = 1, \dots, n$ and denote the corresponding decisions and losses as $u_k^\tau$ and $\xi^\tau$, respectively.
Fix an arm $k$ and update count $n \ge1$. We prove the inequalities in~\eqref{eq:conf_intervals}, and then apply a union bound over all $(k,n)$.

\paragraph{Step 1: Lower Confidence Bound.}
By Asm.~\ref{assumption:expert_regret} (anytime $(U_k,\delta)$-boundedness), with probability at least $1-\delta_{\mathrm{arm}}$, simultaneously for all $n$,
\begin{equation}
\label{eq:anytime-inner}
\frac{1}{n}\sum_{\tau=1}^n \ell(u_k^\tau, \xi^\tau)
-\frac{U_k(n,\delta_{\mathrm{arm}})}{n} \le L_k^\star.
\end{equation}

Evaluating this expression at $n = n_k^t$ and using the definition of $\LCB_k(t,\delta)$ (Def.~\ref{def:ucb_lcb}),
we obtain $\LCB_k(t,\delta) \le L_k^\star$.

\paragraph{Step 2: Upper Confidence Bound.}
Condition on the history $\mathcal H_k^{n-1} = \{(u_k^\tau, \xi^\tau) : \tau \in [n-1]\}$ we get $u_k^n$ is measurable and $\xi^n$ is independent, so $\mathbb E[\ell(u_k^n, \xi^n)\mid \mathcal H_k^{n-1}] = L(u_k^n)$.
Thus $\{\ell(u_k^\tau, \xi^\tau) - L(u_k^\tau)\}_{\tau=1}^n$ is a bounded martingale difference sequence.
By Lemma~\ref{lemma:azuma_adaptive} applied with sample size $n$ and confidence level $\delta_n$, with probability at least $1-\delta_n$,
\begin{equation}
\label{eq:azuma-adapt}
\Biggl|\frac{1}{n}\sum_{\tau=1}^n \ell(u_k^\tau, \xi^\tau) - \frac{1}{n}\sum_{\tau=1}^n L(u_k^\tau)\Biggr|
\le H(n,\delta_n).
\end{equation}
By Asm.~\ref{assumption:wrapper}, for the decision built from the history of arm $k$, i.e. $u^n = \upsilon(\{u_k^\tau: \tau \in [n]\})$, we have $\mathbb E_{u^n}[L(u^n)\mid \{u_k^\tau: \tau \in [n]\}] \leq \frac{1}{n} \sum_{\tau = 1}^n L(u_k^\tau)$. Combining this with \eqref{eq:azuma-adapt} gives
\begin{equation*}
\mathbb E_{u^n}[L(u^n)\mid \{u_k^\tau: \tau \in [n]\}]\le L_{k}(n) + H(n,\delta_n).
\end{equation*}

Since $L_k^\star=\min_{u\in\mathbf U_k}L(u)\le L(u_k^\tau)$ for every $\tau$, $L_k^\star \le \frac1n\sum_{\tau=1}^n L(u_k^\tau)$.
Combining this inequality with \eqref{eq:azuma-adapt} gives
$L_k^\star \le L_k(n)+H(n,\delta_n)$.

Evaluating right hand side expression at $n = n_k^t$ and using the definition of $\UCB_k(t,\delta)$ (Def.~\ref{def:ucb_lcb}),
we obtain $\mathbb E[L(\upsilon(H_k^t))\mid H_k^t]\le \UCB_k(t,\delta)$.
\paragraph{Step 3: Union Bound.}
Finally, we bound the probability of the concentration event $\mathcal{E}_{\delta}$ \eqref{eq:conf_intervals}.
There are two types of events: (i) anytime arm guarantees \eqref{eq:anytime-inner},  and (ii) martingale concentration \eqref{eq:azuma-adapt}.
For (i), each arm contributes at most $\delta_{\mathrm{arm}}$, so over $K$ arms the total failure probability is $\le K\delta_{\mathrm{arm}} = \delta/2$.
For (ii) we allocated $\delta_n$ per event. Since
$\sum_{k=1}^K\sum_{n=1}^\infty \delta_n
= \sum_{k=1}^K \sum_{n=1}^\infty \frac{\delta}{3.5Kn^2}
\le \frac{\delta}{3.5}\cdot \frac{\pi^2}{6} < \delta/2,$

concentration bounds hold simultaneously for all $(k,n)$ with probability at least $1-\delta/2$.
Thus the overall failure probability is at most $\delta$, and the concentration event $\mathcal{E}_{\delta}$ holds with probability at least $1-\delta$.
\end{proof}

\subsection{OCO-style regret}
\label{sec:oco_style_bounds}
In this subsection we derive variance-adaptive confidence bounds for experts whose regret has the OCO form \eqref{eq:regret_oco}; they replace the bounds of Definition~\ref{def:ucb_lcb} in this case. Throughout, fix an expert $k$ and let
\begin{equation*}
\x_n(\delta) := \log \frac{1}{\delta} + 2 \log\left(2 + \log_2 n\right), \qquad n\ge 1.
\end{equation*}

\textbf{Lower bound.} Denote $G(n, \delta) := \frac{2\x_n(\delta)}{3 n}$ and define, for any $t\ge 1$,
    \[
    \LCB_{k}(t, \delta)  := L_{k}(t) - \sqrt{6G(n^t_k, \delta)L_{k}(t)}  - G(n^t_{k}, \delta) - \tfrac{\overline{U}_k(n^t_k, \delta)
    }{n^t_k},
    \]
where $\overline{U}_k(n,\delta)$ is an anytime high-probability bound on the OCO-style regret \eqref{eq:regret_oco} (as in Lemma~\ref{lem:prefix_to_expected}). By Lemma~\ref{lemma:lcb_oco} below, for any fixed $u_k \in \mathbf{U}_k$ (in particular for $u_k^\star$, and hence for $L_k^\star$), with probability at least $1-2\delta$ it holds $\LCB_k(t,\delta) \le L(u_k)$ simultaneously for all $t \ge 1$ with $n_k^t \ge 1$.

\textbf{Upper bound.}
    For any $t \ge 1$ and $\delta \in (0, e^{-1}]$ define
  \begin{align*}
  & \UCB_{k}(t, \delta) := L_{k}(t) \\
  & + \frac{9\log\frac{1}{\delta}}{2n^t_k}\left(7 + \log \log\frac{1}{\delta} + \log(1 + 4n^t_kL_{k}(t)) \right)  \\
   &+  \frac{1}{n^t_k}\sqrt{2\log\frac{1}{\delta} \left(1+ 4n^t_kL_{k}(t)\right) \left(1 + \frac{1}{2}\log \left(1 + 4n^t_kL_{k}(t) \right) \right)}.
   \end{align*}
By Lemma~\ref{lemma:ucb} below, with probability at least $1-\delta$, $\UCB_k(t,\delta)$ upper-bounds the average expected loss $\frac{1}{n^t_k}\sum_{\tau\in I_k(t)} L(u_k^\tau)$ simultaneously for all $t \ge 1$ with $n_k^t \ge 1$; by Assumption~\ref{assumption:wrapper} it therefore upper-bounds the conditional expected loss of the smoothed decision $\upsilon(\{u_k^\tau : \tau \in I_k(t)\})$.

\begin{lemma}[LCB for OCO-style regret]
\label{lemma:lcb_oco}

Let Assumption~\ref{assumption:stochastic_losses} hold. Consider an expert $k$ as described in Section~\ref{sec:examples_parametric} and suppose that it provides an anytime high-probability bound $\overline{U}_k(n, \delta)$ on the OCO-style regret $\overline{R}_k(t)$ defined in \eqref{eq:regret_oco}. That is, for any fixed $\delta \in (0,1)$:
$
    \mathbb{P}\left( \forall t \ge 1, \quad \overline{R}_k(t) \le \overline{U}_k(n_k^t, \delta) \right) \ge 1 - \delta.
$

Fix $u_k \in \mathbf{U}_k$ and let $L(u_k) := \mathbb{E}\,\ell(u_k, \xi)$. Then for any $\x > 0$, with probability at least $1-2e^{-\x}$, simultaneously for all $t \ge 1$ with $n_k^t \ge 1$,
\begin{equation*}
{L(u_k)}  \ge L_{k}(t) - 2\sqrt{\frac{\x_{n^t_k}}{n^t_k}L_{k}(t)} - \frac{2\x_{n^t_k}}{3n^t_k} - \frac{\overline{U}_k(n^t_k, e^{-\x})}{n^t_k},
\end{equation*}
where $\x_n := \x + 2 \log\left(2 + \log_2 n\right) $ for any $n\ge 1$.

\end{lemma}

\begin{proof}
Let $\tau_1 < \tau_2 < \dots$ be the update rounds of expert $k$ and set $X_i := \ell(u_k, \xi^{\tau_i}) - L(u_k)$. Since $u_k$ is fixed and the draw $\xi^{\tau_i} \sim D$ is independent of the past, $X_i$ is conditionally centered, $X_i \le 1$ almost surely, and its conditional variance satisfies
$\mathbb{E}[X_i^2 \mid \mathcal{G}_{i-1}] = \mathrm{Var}\,(\ell(u_k, \xi)) \le \mathbb{E}\,\ell(u_k, \xi)^2 \le \mathbb{E}\,\ell(u_k, \xi) = L(u_k)$, because $\ell \in [0,1]$ (here $\mathcal{G}_{i-1}$ denotes the $\sigma$-algebra generated by all randomness up to round $\tau_i - 1$).

\paragraph{Step 1: Anytime Bernstein inequality.}
For a fixed horizon $m \ge 1$ and any $\tilde \x > 0$, Freedman's inequality in its maximal form (see, e.g., \cite{de2009self}) applied to the partial sums of $X_i$ gives, with probability at least $1 - e^{-\tilde \x}$,
\begin{equation}
\label{eq:lcb_freedman}
\max_{s\le m}\sum_{i=1}^s X_i \le \sqrt{2 m L(u_k)\,\tilde \x} + \frac{2\tilde \x}{3}.
\end{equation}
Indeed, Freedman's inequality bounds the probability of the complementary event by $\exp\bigl(-\varepsilon^2/(2(mL(u_k) + \varepsilon/3))\bigr)$ with $\varepsilon = \sqrt{2mL(u_k)\tilde\x} + 2\tilde\x/3$, and one checks that $\varepsilon^2 \ge 2\tilde\x\bigl(mL(u_k) + \varepsilon/3\bigr)$.
For each integer $j \ge 0$, apply \eqref{eq:lcb_freedman} with horizon $m = 2^{j+1}$ and confidence $\tilde \x_j := \x + 2\log(2 + j)$. Since
\begin{equation*}
\sum_{j\ge 0} e^{-\tilde \x_j} = e^{-\x} \sum_{j\ge 0} \frac{1}{(2+j)^2} = \Bigl(\frac{\pi^2}{6} - 1\Bigr) e^{-\x} \le e^{-\x},
\end{equation*}
the union bound gives, with probability at least $1-e^{-\x}$, the following statement simultaneously for all $s \ge 1$: choosing $j$ with $2^j \le s < 2^{j+1}$, so that $2^{j+1} \le 2s$ and $\tilde \x_j \le \x_s$,
\begin{equation*}
\frac{1}{s}\sum_{i=1}^s X_i
\le \frac{1}{s}\Bigl(\sqrt{2^{j+2} L(u_k)\,\tilde \x_j} + \frac{2\tilde \x_j}{3}\Bigr)
\le 2\sqrt{\frac{L(u_k)\,\x_s}{s}} + \frac{2\x_s}{3s},
\end{equation*}
where we used $\sqrt{2^{j+2}} \le 2\sqrt{s}$. Equivalently, for all $t \ge 1$ with $n = n^t_k \ge 1$,
\begin{equation}
\label{eq:lcb_bern anytime}
\frac{1}{n}\sum_{\tau \in I_k(t)} \ell(u_k, \xi^{\tau}) \le L(u_k) + 2\sqrt{\frac{L(u_k)\,\x_n}{n}} + \frac{2\x_n}{3n}.
\end{equation}

\paragraph{Step 2: Regret comparison.}
On the regret event (which for the choice $\delta = e^{-\x}$ holds with probability at least $1-e^{-\x}$), simultaneously for all $t \ge 1$ with $n = n^t_k$,
\begin{equation*}
n L_k(t) = \sum_{\tau\in I_k(t)} \ell(u_k^\tau, \xi^\tau) \le \min_{u\in\mathbf{U}_k} \sum_{\tau\in I_k(t)} \ell(u,\xi^\tau) + \overline{U}_k(n,e^{-\x}) \le \sum_{\tau\in I_k(t)} \ell(u_k,\xi^\tau) + \overline{U}_k(n,e^{-\x}).
\end{equation*}
By the union bound, this event and the event of Step 1 hold simultaneously with probability at least $1 - 2e^{-\x}$, and on their intersection \eqref{eq:lcb_bern anytime} yields, for all $t \ge 1$ with $n_k^t \ge 1$,
\begin{equation}
\label{eq:lcb_oco_bern}
L_k(t) \le L(u_k) + 2\sqrt{\frac{L(u_k)\,\x_n}{n}} + \frac{2\x_n}{3n} + \frac{\overline{U}_k(n,e^{-\x})}{n}, \qquad n = n^t_k.
\end{equation}

\paragraph{Step 3: Inversion.}
Set $v := \sqrt{L(u_k)}$, $c := 2\sqrt{\x_n/n}$ and $B := \frac{2\x_n}{3n} + \frac{\overline{U}_k(n,e^{-\x})}{n}$, so that \eqref{eq:lcb_oco_bern} reads $L_k(t) \le (v + c/2)^2 - c^2/4 + B$. If $L_k(t) \le B$, the claimed lower bound on $L(u_k)$ is non-positive and therefore holds trivially. Otherwise $v \ge \sqrt{L_k(t) - B + c^2/4} - c/2 \ge 0$, and
\begin{equation*}
L(u_k) = v^2 \ge L_k(t) - B + \frac{c^2}{2} - c\sqrt{L_k(t) - B + \frac{c^2}{4}} \ge L_k(t) - B - c\sqrt{L_k(t)},
\end{equation*}
using $\sqrt{L_k(t) - B + c^2/4} \le \sqrt{L_k(t)} + c/2$. Substituting $c = 2\sqrt{\x_n/n}$ and $B$ gives the claim.
\end{proof}

\begin{lemma}[UCB]
\label{lemma:ucb}
Let Assumption~\ref{assumption:stochastic_losses} hold and let $u^{\tau}_k$ be the decision taken by expert $k$ at round $\tau \in I_k(t)$. Then for any $\x \ge 1$, with probability at least $1-e^{-\x}$, simultaneously for all $t \ge 1$ with $n_k^t \ge 1$
  \begin{multline*}
   \frac{1}{n^t_k}\sum_{\tau\in I_k(t)}  {L(u^{\tau}_k)} \le L_{k}(t)
   + \frac{9\x}{2n^t_k}\left(7 + \log \x + \log(1 + 4n^t_kL_{k}(t)) \right)  +\\
   +  \frac{1}{n^t_k}\sqrt{2\x (1+ 4n^t_kL_{k}(t)) \left(1 + \frac{1}{2}\log \left(1 + 4n^t_kL_{k}(t) \right) \right)}.
   \end{multline*}
\end{lemma}
\begin{proof}
Let $\tau_1 < \tau_2 < \dots$ be the update rounds of expert $k$, let $\mathcal{F}_t$ be the $\sigma$-algebra generated by all randomness up to round $t$, and set $\mathcal{G}_j := \mathcal{F}_{\tau_{j+1}-1}$ (with $\mathcal{G}_0 := \mathcal{F}_{\tau_1 - 1}$). Set $X_j := \ell(u^{\tau_j}_k, \xi^{\tau_j}) \in [0,1]$. Then $X_j$ is $\mathcal{G}_j$-measurable and $u_k^{\tau_j}$ is $\mathcal{G}_{j-1}$-measurable; since $\xi^{\tau_j} \sim D$ is drawn independently of $\mathcal{G}_{j-1}$,
\begin{equation*}
\mu_j := \mathbb{E}[X_j \mid \mathcal{G}_{j-1}] = L(u_k^{\tau_j}).
\end{equation*}
For $n = n^t_k$ we have $S_n := \sum_{j=1}^n X_j = n L_k(t)$ and $U_n := \sum_{j=1}^n \mu_j = \sum_{\tau \in I_k(t)} L(u_k^\tau)$. Applying the second bound of Lemma~\ref{lemma:aux} to the process $(X_j, \mathcal{G}_j)$ and dividing by $n^t_k$ gives the claim.
\end{proof}

\begin{lemma}
\label{lemma:aux}
    Let $X_t \in [0, 1]$ be a process adapted to a filtration $\mathcal{F}_t$.
    Set $\mu_t := \mathbb{E}[X_t| \mathcal{F}_{t-1}]$, $S_t = \sum^t_{i=1} X_i$ and $U_t = \sum^t_{i=1} \mu_i$.  Then for any $\x \ge 1$ with probability at least $1 - e^{-\x}$ and for all $t$ simultaneously it holds
    \begin{equation}
       \label{def:aux_bound_stopping}
    \left|S_t - U_t \right| \le
    \sqrt{2\x \left(S_t + 3U_t + 1\right)\left(1 + \frac{1}{2}\log(S_t + 3U_t + 1) \right)}.
    \end{equation}
   Moreover, on the same event,
   \begin{multline}
   \label{eq:aux_inverted}
   U_t \le S_t + \frac{9\x}{2}\left(7 + \log \x + \log(1 + 4S_t) \right)
   + \sqrt{2\x (1+ 4S_t) \left(1 + \frac{1}{2}\log \left(1 + 4S_t \right) \right)}.
   \end{multline}
\end{lemma}

\begin{proof}
\textbf{Step 1: Exponential supermartingale.}
We first show that for every $\lambda \in \mathbb{R}$ the process
\begin{equation*}
M_t(\lambda) := \exp\left\{\lambda (S_t - U_t) - \frac{\lambda^2}{2}(S_t + 3U_t) \right\}
\end{equation*}
is a supermartingale with $\mathbb{E}\,M_t(\lambda) \le 1$. This is the exponential inequality underlying Theorem~\ref{thm:bercu-touati} (Theorem~9.21 of \cite{de2009self}, a result of Bercu and Touati), which establishes such supermartingales for general locally square-integrable martingales with compensator $\langle M \rangle + [M]$; we verify the one-step increment inequality for our bounded increments, which yields the empirical compensator $S_t + 3U_t$ directly. The one-step multiplicative increment is $\mathbb{E}\bigl[e^{g(X_t)} \,\big|\, \mathcal{F}_{t-1}\bigr]$ with the affine function $g(y) := \lambda(y - \mu_t) - \frac{\lambda^2}{2}(y + 3\mu_t)$. Since $e^{g}$ is convex and $X_t \in [0,1]$ has conditional mean $\mu_t$,
\begin{equation*}
\mathbb{E}\bigl[e^{g(X_t)} \,\big|\, \mathcal{F}_{t-1}\bigr] \le (1-\mu_t)e^{g(0)} + \mu_t e^{g(1)} = e^{-\mu_t c}\left(1 + \mu_t\bigl(e^{d} - 1\bigr)\right),
\end{equation*}
where $c := \lambda + \frac{3}{2}\lambda^2$ and $d := \lambda - \frac{1}{2}\lambda^2$. Since $e^{\mu_t c} \ge 1 + \mu_t c$, the increment is at most $1$ whenever $e^{d} \le 1 + c$, i.e. $d \le \log(1+c)$ (note $1 + c = 1 + \lambda + \frac32 \lambda^2 > 0$ for all $\lambda$). Consider $q(\lambda) := \log\bigl(1 + \lambda + \frac32\lambda^2\bigr) - \lambda + \frac12\lambda^2$: it satisfies $q(0) = 0$ and
\begin{equation*}
q'(\lambda) = \frac{\lambda\left(3 - \frac{\lambda}{2} + \frac{3\lambda^2}{2}\right)}{1 + \lambda + \frac32\lambda^2},
\end{equation*}
whose sign coincides with $\mathrm{sign}(\lambda)$ since $3\lambda^2/2 - \lambda/2 + 3 > 0$ everywhere. Hence $q \ge 0$, i.e. $d \le \log(1+c)$ for all $\lambda \in \mathbb{R}$, and $M_t(\lambda)$ is a supermartingale.

\textbf{Step 2: Anytime bound via the method of mixtures.}
Step 1 shows that the pair $(A, B^2) := (S_t - U_t,\; S_t + 3U_t)$ satisfies the canonical assumption $\mathbb{E} \exp\bigl\{\lambda A - \tfrac{\lambda^2}{2} B^2\bigr\} \le 1$ for all $\lambda \in \mathbb R$, at every $t \ge 1$ and --- by optional sampling of the nonnegative supermartingale $M_t(\lambda)$ at $\tau \wedge n$, followed by Fatou's lemma --- at any stopping time $\tau$. Let $\tau$ be the first time $t$ at which \eqref{def:aux_bound_stopping} is violated; on the event $\{\tau < \infty\}$,
\begin{equation*}
|S_\tau - U_\tau| \;>\; \sqrt{2\x \left(S_\tau + 3U_\tau + 1\right)\left(1 + \frac{1}{2}\log(S_\tau + 3U_\tau + 1) \right)}.
\end{equation*}
Applying Corollary~\ref{cor:mixture-12-5} (Corollary~12.5 of \cite{de2009self}) with $y = 1$ --- whose condition $x \ge \sqrt{2}$ is met by $x := \sqrt{2\x}$, since $\x \ge 1$ --- gives $\mathbb{P}(\tau < \infty) \le e^{-\x}$, which is exactly \eqref{def:aux_bound_stopping}.

\textbf{Step 3: Inversion.}
The bound \eqref{eq:aux_inverted} follows from Lemma~\ref{lemma:aux2} applied to \eqref{def:aux_bound_stopping}.
\end{proof}

\begin{proof}[Proof of Lemma~\ref{prop:m_mean_regret}]
\label{proof:m_mean_regret}
Let $S^\star\in\arg\min_{|S|=M}\frac1M\sum_{k\in S}L_k^\star$.
On $\mathcal E_\delta$, $\LCB_k(t,\delta)\le L_k^\star\le \UCB_k(t,\delta)$ for all $k,t$.
Since $S_t$ minimizes the sum of LCBs,
\[
\frac1M\sum_{k\in S_t}\LCB_k(t,\delta)
\le
\frac1M\sum_{k\in S^\star}\LCB_k(t,\delta)
\le
\frac1M\sum_{k\in S^\star}L_k^\star
=\overline L^\star.
\]
Therefore,
\[
\frac1M\sum_{k\in S_t}L_k^\star-\overline L^\star
\le
\frac1M\sum_{k\in S_t}(L_k^\star-\LCB_k(t,\delta))
\le
\frac1M\sum_{k\in S_t}(\UCB_k(t,\delta)-\LCB_k(t,\delta)).
\]
Summing over $t=1,\ldots,T$ and regrouping selected experts by update count gives
\[
\mathrm{Reg}_M(T)
\le
\frac1M\sum_{k=1}^K\sum_{r=1}^{N_k(T)}
[\UCB_k^r(\delta)-\LCB_k^r(\delta)].
\]
\end{proof}

\section{Lower Bounds}
\label{sec:lower-bounds}

We establish minimax lower bounds for our problem setup.
The proof combines information-theoretic arguments (as in \cite{seldin2014prediction})
with the internal hardness of heavy-tailed bandits \cite{bubeck2013bandits}.
The main idea is to construct a family of perturbed games in which the optimal expert is
separated from every other expert by a per-round gap $\varepsilon+\Delta$, where
$\varepsilon$ is an additional cross-expert separation and $\Delta$ is the internal gap of
each expert's own two-armed subproblem.
Any round in which a suboptimal expert is deployed therefore costs at least
$\varepsilon+\Delta$, so the regret is controlled by the number of deployments of
non-optimal experts. We bound the latter by relating the probability of identifying the
optimal expert to KL divergences via Pinsker's inequality, using a symmetric
\emph{null game} as the reference measure.
The heavy-tailed noise is essential: it makes the perturbed games statistically hard to
distinguish, since each observation of an expert contributes a KL divergence of order only
$(\varepsilon+\Delta)^{(1+\beta)/\beta}$ (Lemma~\ref{lem:ht-family}); this is what produces
the exponent $\alpha = 1/(1+\beta)$ in the final rate.

\paragraph{Bandit setting, protocol, and regret conventions.}
For concreteness, we represent experts by independent two-armed stochastic bandit problems (the endpoint $\alpha = 1$, treated separately below, uses bandit experts with finitely many arms).
Each expert $h \in [K]$ has two arms whose losses are random variables with means
$\ell_{h,1}$ and $\ell_{h,2}$, and we write $\ell^\star_h = \min\{\ell_{h,1},\,\ell_{h,2}\}$.
Throughout this section we use the following instantiation of the protocol of
Section~\ref{sec:losses}. At each round $t = 1,\dots,T$:
\begin{enumerate}[leftmargin=2em]
    \item The meta-procedure $\mathcal P$ chooses $S_t \subseteq [K]$ with $|S_t| \le M$ and the deployed expert $H_t \in S_t$, as a function of the past observations and its own randomization.
    \item Each expert $k \in S_t$ plays an arm $I_t^k \in \{1,2\}$, chosen by its algorithm $\mathscr A_k$ as a function of the losses it observed on its own previous update rounds and its internal randomization.
    \item The deployed loss $\ell(u^t,\xi^t)$ is a fresh draw from the loss distribution of arm $I_t^{H_t}$ of the deployed expert. We denote its conditional mean given the history by $\ell_t^{H_t} := \ell_{H_t, I_t^{H_t}}$.
    \item For each $k \in S_t$, the realized loss of arm $I_t^k$ is observed (bandit feedback) and expert $k$ updates its state.
\end{enumerate}
For a two-armed expert $h$, we denote by
\begin{equation}
\label{eq:internal-regret-def}
R_{h}^{\mathrm{in}}(n) =\sum_{j=1}^n \mathbb{E} [\ell_{h,I_j}] - n \cdot \ell_h^\star,
\end{equation}
its (expected) internal regret over its first $n$ update rounds, where $I_j$ is the arm it plays on its $j$-th update round.

\subsection{Class of $\alpha$–hard stochastic tasks}
We first restate the lower bound and introduce the heavy-tailed class of base problems that makes each expert's subproblem intrinsically hard.

\begin{theorem}[Lower bound]
\label{thm:mt-alpha_appendix}
Consider $K$ experts, horizon $T$, and a per-round budget $M$. Fix $\alpha\in[0,1]$. There exists a class $\mathcal F_\alpha$ satisfying Def.~\ref{def:alpha-class}, s.t. if each
expert $k \in [K]$ solves a problem $f_k \in \mathcal F_\alpha$, then, for sufficiently small $\sqrt{\frac{K\log K}{MT}}$ and for any learning algorithm $\mathscr{A}_k$ and meta-procedure $\mathcal{P}$
\begin{equation}
\label{eq:mt-alpha-internal}
\sup_{f_{k} \in \mathcal{F}_{\alpha},\,k \in [K]} \mathbb{E}\,\mathrm{Reg}(T) \ge c_1\,\sqrt{\frac{K\,T}{M}}
\;+\; c_2\,T^\alpha \left(\frac{K}{M} \right)^{1-\alpha},
\end{equation}
where $c_1,c_2>0$ are constants depending only on $\alpha$.
\end{theorem}
\begin{proof}
The proof occupies the remainder of this section: the heavy-tailed arm family and the hardness of the base class (Lemma~\ref{lem:ht-family} and Corollary~\ref{cor:heavy-tail-lb}), the composite game construction, the regret decomposition (Step~1), the averaged identification bound (Step~2), the KL computation (Step~3, Lemma~\ref{lem:kl-total}), and the final parameter choices (Step~4), with the endpoint $\alpha=1$ treated separately.
\end{proof}

We build $\mathcal{F}_{\alpha}$ from heavy-tailed multi-armed bandits \cite{bubeck2013bandits},
whose arm losses have bounded $(1+\beta)$-moments: each arm $i$ provides losses with mean $\mu_i$ and
\begin{equation}
\label{eq:noise_level}
    \mathbb{E}_{X\sim \nu_i} |X - \mu_i|^{1 + \beta}\leq u,
\end{equation}
for some $u>0$ and $\beta\in(0,1]$.
Our construction uses the following explicit family of arm distributions, whose key property is that the KL divergence between two of its members decays \emph{faster than quadratically} in their mean gap.

\begin{lemma}[A heavy-tailed family with super-quadratic KL]
\label{lem:ht-family}
Fix $\beta \in (0,1]$ and $\Delta \in (0, \tfrac{1}{36}]$, and set
\begin{equation*}
A := (12\Delta)^{-1/\beta},
\qquad
\eta := \tfrac12 A^{-(1+\beta)},
\qquad
b := \eta A = 6\Delta .
\end{equation*}
For $\theta \in [0, 3\Delta]$, define
\begin{equation*}
\nu_\theta \;=\; \Bigl(1 - \eta - \frac{\theta}{A}\Bigr)\,\delta_{0} \;+\; \Bigl(\eta + \frac{\theta}{A}\Bigr)\,\delta_{A},
\end{equation*}
where $\delta_x$ denotes the point mass at $x$. Then:
\begin{enumerate}[leftmargin=2em,label=(\roman*)]
\item $\nu_\theta$ is supported on $\mathbb R_+$, has mean $b + \theta$, and satisfies \eqref{eq:noise_level} with $u=1$;
\item for every $\theta \in [0, 2\Delta]$ and every shift $s \in [\Delta,\, 3\Delta]$ with $\theta + s \le 3\Delta$, both KL divergences are finite and satisfy
\begin{equation*}
\max\Bigl\{\mathrm{KL}(\nu_{\theta+s} \,\|\, \nu_\theta),\; \mathrm{KL}(\nu_{\theta} \,\|\, \nu_{\theta+s})\Bigr\} \;\le\; C_\beta\, s^{\frac{1+\beta}{\beta}},
\qquad C_\beta := 4\cdot 12^{1/\beta}.
\end{equation*}
\end{enumerate}
\end{lemma}

From the analysis of \cite{bubeck2013bandits}, the associated class is hard for single-expert learning:

\begin{corollary}[Heavy-tailed tasks are $\alpha$-hard; from \cite{bubeck2013bandits}, Thm.~2]
\label{cor:heavy-tail-lb}
Let $\beta \in (0,1]$ and $\alpha = \tfrac{1}{1+\beta}$. The class $\mathcal F_\alpha$ of two-armed bandit problems whose arm losses satisfy \eqref{eq:noise_level} with $u = 1$ admits an $\alpha$-regret lower bound in the sense of Def.~\ref{def:alpha-class}.
More precisely, with $(\nu_0, \nu_\Delta)$ as in Lemma~\ref{lem:ht-family}, for any algorithm and any horizon $n$,
\begin{equation}
\label{eq:heavy-lb-general}
\max\Bigl\{ R^{\mathrm{in}}_{(\nu_0,\nu_\Delta)}(n),\; R^{\mathrm{in}}_{(\nu_\Delta,\nu_0)}(n) \Bigr\}
\;\ge\;
\frac{n\Delta}{4}\Bigl(1 - c_\beta \sqrt{n\Delta^{\frac{1+\beta}{\beta}}}\Bigr),
\end{equation}
where $R^{\mathrm{in}}_{(\nu,\nu')}(n)$ denotes the internal regret \eqref{eq:internal-regret-def} on the instance with arms $(\nu,\nu')$, and $c_\beta>0$ depends only on $\beta$.
For fixed $n$, choosing
\begin{equation}
\label{eq:heavy-lb-opt}
\Delta = \min\Bigl\{c_0\, n^{-\frac{\beta}{1+\beta}},\; \tfrac{1}{36}\Bigr\}
\end{equation}
with sufficiently small $c_0>0$ yields
\begin{equation}
\label{eq:heavy-lb-final}
\sup_{f \in \mathcal F_\alpha} R^{\mathrm{in}}(n) \ge c'\,n^{\frac{1}{1+\beta}} \;=\; c'\,n^{\alpha},
\end{equation}
for some constant $c'>0$ depending only on $\beta$.
\end{corollary}

For heavy-tailed bandits this yields $\alpha=\tfrac{1}{1+\beta}\in[0.5,1)$. The lower-bound argument below also applies to $\alpha\in[0,0.5)$: there we simply take $\beta = 1$ (the class is then $\tfrac12$-hard, hence $\alpha$-hard for every $\alpha \le \tfrac12$ by \eqref{eq:heavy-lb-final}), and the exploration term already dominates the claimed rate.

\begin{remark}[Loss boundedness in the lower bound]
\label{rem:lb-boundedness}
For $\alpha > 1/2$ the class $\mathcal F_\alpha$ necessarily contains unbounded losses: for bandit classes with losses in $[0,1]$, algorithms such as UCB achieve $O(\sqrt{n})$ regret on every instance, so no such class can satisfy Def.~\ref{def:alpha-class} with $\alpha > 1/2$. Accordingly, for $\alpha>1/2$ the lower bound is stated for non-negative losses satisfying the moment condition~\eqref{eq:noise_level} (cf. the remark on bounded losses in Section~\ref{sec:setup}). For $\alpha \le 1/2$ the argument below can equally be carried out with Bernoulli losses in $[0,1]$ (the only property used is $\mathrm{KL}(\nu_{\theta+s}\|\nu_\theta) \le C(\varepsilon+\Delta)^2$, which Bernoulli arms with means bounded away from $0$ and $1$ satisfy), and for $\alpha = 1$ the construction below uses losses in $[0,1]$.
\end{remark}

\paragraph{The endpoint $\alpha = 1$.}
For $\alpha = 1$ the two-armed heavy-tailed class is not suitable: any class of two-armed stochastic bandits admits algorithms with uniformly sublinear regret, so it cannot satisfy Def.~\ref{def:alpha-class} with $\alpha = 1$. Instead, we take $\mathcal F_1$ to be the class of stochastic bandit problems with losses in $[0,1]$ and an arbitrary finite number of arms.
This class satisfies Def.~\ref{def:alpha-class} with $\alpha = 1$ by a standard hide-one-arm argument: for any algorithm and any $T \ge 1$, consider the instance with $N = T+1$ arms, all with loss $\mathrm{Ber}(\tfrac12)$ except a single arm $j^\star$, placed uniformly at random, with loss $\mathrm{Ber}(\tfrac12-\delta)$ and $\delta = \tfrac14$.
Comparing each instance with the all-$\mathrm{Ber}(\tfrac12)$ instance via Pinsker's inequality (as in the proof of Corollary~\ref{cor:heavy-tail-lb}), the expected number of plays of $j^\star$ by any algorithm in $T$ rounds, averaged over the placement of $j^\star$, is at most
\begin{equation*}
\frac{T}{N} + T\sqrt{\frac{c_\delta T}{2N}}
\;\le\;
\frac{7T}{10},
\qquad
c_\delta := \mathrm{KL}\bigl(\mathrm{Ber}(\tfrac12)\,\big\|\,\mathrm{Ber}(\tfrac12-\delta)\bigr) \approx 0.144,
\end{equation*}
hence $\sup_{f\in\mathcal F_1}\mathbb E R^{\mathrm{in}}(T) \ge \delta T/4$.
For the meta-bound, in game $h$ give expert $h$ $N = T+1$ arms: the hidden arm $j^\star$ with mean $\tfrac12 - \Delta$, $\Delta = \tfrac18$, and all remaining arms with mean $\tfrac12$, while every expert $h' \neq h$ has all arms with mean $\tfrac12$.
The optimal loss is $L^\star = \tfrac12-\Delta$, and any round in which the deployed arm is not arm $j^\star$ of expert $h$ costs at least $\Delta$, whether $H_t = h$ or not.
Rounds with $H_t = h$ and arm $j^\star$ played are plays of $j^\star$ by expert $h$ (since $H_t \in S_t$), so their number is at most the total number of plays of $j^\star$ by expert $h$, which by the hide-one-arm bound is at most $7T/10$ in expectation over the placement of $j^\star$ (the bound applies to the joint meta-procedure--expert system, since the placement of $j^\star$ is independent of the algorithms and is only revealed through the played arms of expert $h$).
Therefore
\begin{equation*}
\sup_{f_k \in \mathcal F_1,\,k\in[K]} \mathbb E\,\mathrm{Reg}(T)
\;\ge\;
\Delta\Bigl(T - \frac{7T}{10}\Bigr)
\;\ge\;
\frac{\Delta T}{4}
\;\ge\;
\frac{\Delta}{8}\biggl(\sqrt{\frac{KT}{M}} + T\biggr),
\end{equation*}
using $\sqrt{KT/M} \le T$ for $K \le MT$. This matches the claimed rate at $\alpha = 1$, so in what follows we treat $\alpha \in [0,1)$.

\subsection{Construction of the composite game}

Let $K$ be the number of experts, $M$ the per-round optimization budget, and $T$ the horizon.
We construct $K$ perturbed games together with one symmetric \emph{null game}.
The construction depends on two small parameters: $\Delta>0$ (internal gap) and $\varepsilon>0$ (extra cross-expert separation), and we always assume $\varepsilon \le \Delta$.
All arm distributions are taken from the family of Lemma~\ref{lem:ht-family} built at scale $\Delta$, i.e. with $A = (12\Delta)^{-1/\beta}$.

\paragraph{Perturbed games.}
In the $h$-th perturbed game, expert $h$ faces the two-armed bandit with arm distributions
$(\nu_{0}, \nu_{\Delta})$, i.e. means $\ell_{h,1} = b$ and $\ell_{h,2} = b+\Delta$,
while any expert $h'\neq h$ faces $(\nu_{\varepsilon+\Delta}, \nu_{\varepsilon+2\Delta})$, i.e. means
$\ell_{h',1} = b+\varepsilon+\Delta$ and $\ell_{h',2} = b+\varepsilon+2\Delta$.
Thus expert $h$ is uniquely optimal in game $h$: $\ell^\star_h = b$, whereas
$\ell^\star_{h'} = b+\varepsilon + \Delta$ for every $h' \neq h$.

\paragraph{Null game.}
All experts face identical subproblems with arm distributions
$(\nu_{\varepsilon+\Delta}, \nu_{\varepsilon+2\Delta})$, i.e. means
$\ell_{h,1} = b+\varepsilon+\Delta$ and $\ell_{h,2} = b+\varepsilon+2\Delta$ for every $h \in [K]$.

By Lemma~\ref{lem:ht-family}(i), every arm distribution used above satisfies \eqref{eq:noise_level} with $u=1$, so each perturbed game is a valid instance with $f_k \in \mathcal F_\alpha$, $k \in [K]$.

\subsection{Step 1: Regret decomposition}

Fix $h \in [K]$ and write $\mathbb{P}_h$, $\mathbb{E}_h$ for probabilities and expectations in the $h$-th perturbed game.
In game $h$ the optimal expected loss is $L^\star = \ell_h^\star = b$, hence, by the tower property of conditional expectation,
\begin{equation}
R_h(T) \;:=\; \mathbb{E}_h\bigg[\sum_{t=1}^T \bigl(\ell(u^t,\xi^t) - b\bigr)\bigg]
\;=\; \mathbb{E}_h\bigg[\sum_{t=1}^T \bigl(\ell_t^{H_t} - b\bigr)\bigg],
\end{equation}
with $\ell_t^{H_t} = \ell_{H_t, I_t^{H_t}}$ the conditional mean of the deployed arm given the history.
Define the empirical deployment frequencies
\begin{equation}
\hat q_h = \frac{1}{T}\sum_{t=1}^T \mathbf{1}\{H_t = h\},
\qquad
J \sim \hat q,
\end{equation}
where $J$ is a random index sampled according to $\hat q$, so that
\begin{equation}
\mathbb{P}_h(J = h)
\;=\;
\mathbb{E}_h\!\left[\frac{1}{T}\sum_{t=1}^T \mathbf{1}\{H_t = h\}\right].
\end{equation}
Whenever $H_t = h' \neq h$, the deployed arm has mean
$\ell_t^{H_t} \in \{b+\varepsilon+\Delta,\, b+\varepsilon+2\Delta\}$;
whenever $H_t = h$, it has mean $\ell_t^{H_t} \in \{b, b+\Delta\} = \{\ell_h^\star,\, \ell_h^\star + \Delta\}$.
Therefore, on every round,
\begin{equation}
\ell_t^{H_t} - b
\;=\;
(\varepsilon+\Delta)\,\mathbf 1\{H_t \neq h\}
\;+\;
\bigl(\ell_t^{H_t} - \ell^\star_{H_t}\bigr)
\;\ge\;
(\varepsilon+\Delta)\,\mathbf 1\{H_t \neq h\},
\end{equation}
where the dropped term $\ell_t^{H_t} - \ell^\star_{H_t} \ge 0$ is the internal suboptimality of the deployed expert.
Summing over rounds and taking expectations,
\begin{equation}
\label{eq:bandit-regret-prob}
R_h(T) \;\ge\; (\varepsilon+\Delta)\,T\,\bigl(1 - \mathbb{P}_h(J=h)\bigr).
\end{equation}
The point is that deploying a non-optimal expert costs $\varepsilon + \Delta$ per round, and not only $\varepsilon$: the internal gap $\Delta$ contributes to the separation between games, and no control of the deployment counts of the experts is needed.
Averaging over the $K$ perturbed games,
\begin{equation}
\label{eq:lb-sup-avg}
\sup_{f_k \in \mathcal F_\alpha,\, k\in[K]} \mathbb{E}\,\mathrm{Reg}(T)
\;\ge\;
\max_h R_h(T)
\;\ge\;
(\varepsilon+\Delta)\,T\,\bigg(1 - \frac{1}{K}\sum_{h=1}^K \mathbb{P}_h(J=h)\bigg).
\end{equation}

\subsection{Step 2: Pinsker and the averaged identification bound}

\begin{lemma}[Pinsker's inequality]
\label{lem:pinsker}
For any $h$ and any event $A$ measurable with respect to the interaction transcript,
\begin{equation}
|\mathbb{P}_h(A) - \mathbb{P}_\emptyset(A)| \le \sqrt{\tfrac{1}{2}\mathrm{KL}(\mathbb{P}_\emptyset\|\mathbb{P}_h)},
\end{equation}
where $\mathbb{P}_\emptyset$ denotes the law under the null game. In particular:
\begin{equation}
    \mathbb{P}_h[J=h] \le \mathbb{P}_\emptyset[J = h] + \sqrt{\tfrac{1}{2}\mathrm{KL}(\mathbb{P}_\emptyset\|\mathbb{P}_h)}
\end{equation}
\end{lemma}
\begin{proof}
This is the classical Pinsker inequality combined with the Bretagnolle–Huber bound on the total variation distance; see, e.g., \cite[Chapter~2]{bubeck2010bandits}.
\end{proof}

Averaging over the games and using $\sum_{h=1}^K \mathbb{P}_\emptyset(J=h) = 1$ (exactly one expert is deployed each round) together with the concavity of the square root, we get:
\begin{equation}
\label{eq:identification_bound}
\frac{1}{K}\sum_{h=1}^K \mathbb{P}_h[J=h]
\le
\frac{1}{K} +
\sqrt{\frac{1}{2K}\sum_{h=1}^K \mathrm{KL}(\mathbb{P}_\emptyset\Vert \mathbb{P}_h)}.
\end{equation}

\subsection{Step 3: KL computation}

\begin{lemma}[KL for $T$ rounds]
\label{lem:kl-total}
Let $\mathbb{P}_\emptyset, \mathbb{P}_h$ denote the laws of the interaction transcript up to round $T$ in the null game and in the $h$-th perturbed game, and let $n_h(T) := \sum_{t=1}^T \mathbf 1\{h \in S_t\}$ be the update count of expert $h$. Assume $\varepsilon \le \Delta$. Then
\begin{equation}
\mathrm{KL}(\mathbb P_\emptyset\|\mathbb P_h) \;\le\; C_\beta\,(\varepsilon+\Delta)^{\frac{1+\beta}{\beta}}\,\mathbb{E}_\emptyset\bigl[n_h(T)\bigr],
\end{equation}
with $C_\beta$ from Lemma~\ref{lem:ht-family}. Moreover,
\begin{equation}
\sum_{h=1}^K \mathrm{KL}(\mathbb P_\emptyset \| \mathbb P_h)
\;\le\; C_\beta\,(\varepsilon+\Delta)^{\frac{1+\beta}{\beta}}\,MT.
\end{equation}
\end{lemma}

\subsection{Step 4: Putting all together}
We assume $M \le K$ (a larger budget is vacuous, since at most $K$ experts can be updated per round) and $K \ge 2$. For $K = 1$ the bound follows directly from Corollary~\ref{cor:heavy-tail-lb}: the protocol forces a deployment on every round, so the single expert is deployed and updated on every round, and on the hard instance $\mathbb E\,\mathrm{Reg}(T) = R^{\mathrm{in}}(T) \ge c'\,T^\alpha \ge \tfrac{c'}{2}\bigl(\sqrt{T/M} + T^\alpha\bigr)$, where the last inequality uses $M = 1$ together with $T^\alpha \ge \sqrt{T}$ for $\alpha \ge 1/2$ and the $\tfrac12$-hardness $R^{\mathrm{in}}(T) \ge c'T^{1/2} \ge \tfrac{c'}{2}\bigl(\sqrt{T} + T^\alpha\bigr)$ for $\alpha \le 1/2$.
Combining \eqref{eq:lb-sup-avg}, \eqref{eq:identification_bound} and Lemma~\ref{lem:kl-total},
\begin{equation}
\label{eq:lb-master}
\sup_{f_k \in \mathcal F_\alpha,\,k\in[K]} \mathbb{E}\,\mathrm{Reg}(T)
\;\ge\;
(\varepsilon+\Delta)\,T\left(1 - \frac{1}{K} - \sqrt{\frac{C_\beta(\varepsilon+\Delta)^{\frac{1+\beta}{\beta}} M T}{2K}}\right).
\end{equation}
We use different parameter choices in the two regimes.

\paragraph{Case $\alpha\le 1/2$.}
Take $\beta = 1$ and set $\Delta=\varepsilon$ with
\begin{equation}
\label{eq:final-eps}
\varepsilon = \gamma\sqrt{\tfrac{K}{MT}}
\end{equation}
for sufficiently small $\gamma>0$. Then $\varepsilon \le \Delta$ holds trivially and
\begin{equation*}
\frac{C_1(\varepsilon+\Delta)^{2} M T}{2K}
=
\frac{4 C_1 \varepsilon^2 M T}{2K}
=
2 C_1 \gamma^2
\;\le\; \frac{1}{16}
\end{equation*}
for $\gamma$ small enough, so the square root in \eqref{eq:lb-master} is at most $\tfrac14$. Hence, for $K \ge 2$,
\begin{equation}
\label{eq:final-ident}
\sup_{f_k}\mathbb{E}\,\mathrm{Reg}(T)
\;\ge\;
2\varepsilon T\left(1-\tfrac{1}{K}-\tfrac{1}{4}\right)
\;\ge\;
\tfrac{\varepsilon T}{2}
\;=\;
\tfrac{\gamma}{2}\sqrt{\tfrac{KT}{M}}.
\end{equation}
Since
\begin{equation*}
\frac{T^\alpha(K/M)^{1-\alpha}}{\sqrt{KT/M}}
=\left(\frac{K}{MT}\right)^{1/2-\alpha}\le 1
\end{equation*}
under the stated smallness condition, the same lower bound implies
\begin{equation}
\label{eq:final-ident-done}
\sup_{f_k}\mathbb{E}\,\mathrm{Reg}(T)
\ge c'_1\left(\sqrt{\tfrac{KT}{M}}+T^\alpha\left(\tfrac{K}{M}\right)^{1-\alpha}\right)
\end{equation}
for a smaller constant $c'_1>0$.

\paragraph{Case $\alpha>1/2$.}
Set $\beta = \frac{1-\alpha}{\alpha} \in (0,1)$, so that $\alpha = \frac{1}{1+\beta}$ and $\frac{1+\beta}{\beta} = \frac{1}{1-\alpha}$. Choose
\begin{equation}
\label{eq:delta-choice-lb}
\Delta = c_0\!\left(\tfrac{K}{MT}\right)^{1-\alpha},
\qquad
\varepsilon = \gamma\,T^{-1/2},
\end{equation}
with $c_0,\gamma>0$ sufficiently small. Since
$T^{1/2}\bigl(\tfrac{K}{MT}\bigr)^{1-\alpha}
=
\bigl(\tfrac{K}{M}\bigr)^{1-\alpha} T^{\alpha-1/2} \ge 1$
for $\alpha \ge \tfrac12$ and $T \ge K/M$, taking $\gamma \le c_0$ ensures $\varepsilon \le \Delta$. Then
\begin{equation*}
\frac{C_\beta(\varepsilon+\Delta)^{\frac{1+\beta}{\beta}} M T}{2K}
\;\le\;
\frac{C_\beta(2\Delta)^{\frac{1}{1-\alpha}} M T}{2K}
\;=\;
\frac{C_\beta (2c_0)^{\frac{1}{1-\alpha}}}{2}
\;\le\;
\frac{1}{16}
\end{equation*}
for $c_0$ small enough (depending only on $\beta$), where we used
$\Delta^{\frac{1}{1-\alpha}} = c_0^{\frac{1}{1-\alpha}}\tfrac{K}{MT}$.
So the square root in \eqref{eq:lb-master} is at most $\tfrac14$, and for $K \ge 2$,
\begin{equation}
\label{eq:final-internal-done}
\sup_{f_k}\mathbb{E}\,\mathrm{Reg}(T)
\;\ge\;
(\varepsilon+\Delta)\,T\left(1-\tfrac{1}{K}-\tfrac{1}{4}\right)
\;\ge\;
\frac{\Delta T}{4}
\;=\;
\frac{c_0}{4}\,T^\alpha\!\left(\tfrac{K}{M}\right)^{1-\alpha}.
\end{equation}
Moreover,
\begin{equation*}
\frac{\sqrt{KT/M}}{T^\alpha(K/M)^{1-\alpha}}
=\left(\frac{K}{MT}\right)^{\alpha-1/2}\le 1
\end{equation*}
under the stated smallness condition. Therefore, after reducing constants,
\begin{equation}
\label{eq:final-lb}
\sup_{f_k}\mathbb{E}\,\mathrm{Reg}(T)
\ge c'_2\left(\sqrt{\tfrac{KT}{M}}+T^\alpha\left(\tfrac{K}{M}\right)^{1-\alpha}\right).
\end{equation}

Combining the two cases with the endpoint $\alpha = 1$ treated above proves Theorem~\ref{thm:mt-alpha_appendix}. The constants may change between displays; they depend only on the fixed exponent $\alpha$ (through $C_\beta$) and remain independent of $K,M,T$, of the expert algorithms, and of the meta-procedure.

\subsection{Proofs of Lemmas for Lower Bounds}
This subsection collects the proofs of the hardness corollary, the heavy-tailed family lemma, and the KL bound stated above.

\begin{proof}[Corollary~\ref{cor:heavy-tail-lb}]
The construction follows the proof of Theorem 2 in \cite{bubeck2013bandits} (see also~\cite[Theorem~2.6]{bubeck2010bandits}); we instantiate it with the family of Lemma~\ref{lem:ht-family}.
Fix $\Delta \in (0,\tfrac{1}{36}]$ and consider the two instances $(\nu_0,\nu_\Delta)$ and $(\nu_\Delta,\nu_0)$, whose arm losses satisfy \eqref{eq:noise_level} with $u=1$ by Lemma~\ref{lem:ht-family}(i).
Fix any algorithm, let $\mathbb P_1,\mathbb P_2$ be the laws of the length-$n$ interaction transcript on the two instances, and let $N_2(n)$ denote the number of plays of arm $2$.
Then
\begin{equation*}
R^{\mathrm{in}}_{(\nu_0,\nu_\Delta)}(n) = \Delta\,\mathbb E_1[N_2(n)],
\qquad
R^{\mathrm{in}}_{(\nu_\Delta,\nu_0)}(n) = \Delta\,\bigl(n - \mathbb E_2[N_2(n)]\bigr).
\end{equation*}
Applying Pinsker's inequality to the event $\{N_2(n) \ge n/2\}$ gives
\begin{equation*}
\mathbb E_1[N_2(n)] + \bigl(n - \mathbb E_2[N_2(n)]\bigr)
\ge
\frac{n}{2}\Bigl(1 - \sqrt{\tfrac12\mathrm{KL}(\mathbb P_1\|\mathbb P_2)}\Bigr).
\end{equation*}
By the chain rule for KL divergence, only the observed arm losses contribute to the divergence between the two transcripts, so
\begin{equation*}
\mathrm{KL}(\mathbb P_1\|\mathbb P_2)
=
\mathbb E_1[N_1(n)]\,\mathrm{KL}(\nu_0\|\nu_\Delta)
+
\mathbb E_1[N_2(n)]\,\mathrm{KL}(\nu_\Delta\|\nu_0)
\le
C_\beta\, n\, \Delta^{\frac{1+\beta}{\beta}},
\end{equation*}
where $N_1(n) = n - N_2(n)$ is the number of plays of arm $1$, and the inequality uses Lemma~\ref{lem:ht-family}(ii) with $\theta=0$, $s=\Delta$ in both directions.
Taking the average over the two instances proves \eqref{eq:heavy-lb-general} with $c_\beta = \sqrt{C_\beta/2}$.
Substituting $\Delta = \min\{c_0 n^{-\beta/(1+\beta)}, \tfrac{1}{36}\}$: if $c_0 n^{-\beta/(1+\beta)} \le \tfrac{1}{36}$, then $n\Delta^{(1+\beta)/\beta} = c_0^{(1+\beta)/\beta}$, so for $c_0$ small enough (depending only on $\beta$) the parenthesis is at least $\tfrac12$ and the bound is at least $\tfrac{c_0}{8} n^{1/(1+\beta)}$; otherwise $\Delta = \tfrac{1}{36}$ and $n\Delta^{(1+\beta)/\beta} < c_0^{(1+\beta)/\beta}$, so the parenthesis is again at least $\tfrac12$ and the bound is at least $\tfrac{n}{288}$. Taking $c' := \min\{\tfrac{c_0}{8}, \tfrac{1}{288}\}$ yields \eqref{eq:heavy-lb-final} for every $n \ge 1$.
Since the instance used at each horizon $n$ belongs to $\mathcal F_\alpha$, display \eqref{eq:heavy-lb-final} is exactly Def.~\ref{def:alpha-class} with exponent $\alpha = 1/(1+\beta)$.
\end{proof}

\begin{proof}[Lemma~\ref{lem:ht-family}]
(i) The support of $\nu_\theta$ is $\{0, A\} \subset \mathbb R_+$, and its mean is $\bigl(\eta + \tfrac{\theta}{A}\bigr) A = \eta A + \theta = b + \theta$.
Write $r := \eta + \theta/A$ for the probability of the atom at $A$; since $\eta = \tfrac12(12\Delta)^{(1+\beta)/\beta} \le \tfrac14$ and $\theta/A = \theta(12\Delta)^{1/\beta} \le 3\Delta(12\Delta) \le \tfrac14$ for $\Delta \le \tfrac{1}{36}$ (using $(12\Delta)^{1/\beta} \le 12\Delta$), we have $r \le \tfrac12$.
For the moment, with $m := b + \theta \le 9\Delta$,
\begin{multline*}
\mathbb E_{X\sim\nu_\theta}|X-m|^{1+\beta}
=
(1-r)\, m^{1+\beta} + r\,(A-m)^{1+\beta}\\
\le
m^{1+\beta} + \eta A^{1+\beta} + \theta A^{\beta}
\le
(9\Delta)^{1+\beta} + \frac12 + \frac{\theta}{12\Delta}
\le 1,
\end{multline*}
since $\theta \le 3\Delta$, $\eta A^{1+\beta} = \tfrac12$, $A^\beta = (12\Delta)^{-1}$, and $(9\Delta)^{1+\beta} \le \tfrac14$ for $\Delta \le \tfrac{1}{36}$.

(ii) With $r := \eta + \theta/A$ and $p := r + s/A = \eta + (\theta+s)/A$, we have $\nu_\theta = (1-r)\delta_0 + r\,\delta_A$ and $\nu_{\theta+s} = (1-p)\,\delta_0 + p\,\delta_A$, and the computation above gives $r \le p \le \tfrac12$.
Using $\mathrm{KL}(\mathrm{Ber}(p)\|\mathrm{Ber}(q)) \le \frac{(p-q)^2}{q(1-q)}$,
\begin{equation*}
\mathrm{KL}(\nu_{\theta+s}\|\nu_\theta)
\le
\frac{(s/A)^2}{r(1-r)}
\le
\frac{2(s/A)^2}{\eta}
=
\frac{4s^2}{A^2}\,A^{1+\beta}
=
4 s^2 (12\Delta)^{\frac{1-\beta}{\beta}}
\le
4\cdot 12^{\frac{1-\beta}{\beta}}\, s^{\frac{1+\beta}{\beta}},
\end{equation*}
where the last step uses $s \ge \Delta$ and $\beta \le 1$, i.e. $s^{(\beta-1)/\beta} \le \Delta^{(\beta-1)/\beta}$.
In the other direction, since $p \ge s/A$ and $1-p \ge \tfrac12$,
\begin{equation*}
\mathrm{KL}(\nu_{\theta}\|\nu_{\theta+s})
\le
\frac{(s/A)^2}{p(1-p)}
\le
\frac{2(s/A)^2}{p}
\le
\frac{2(s/A)^2}{s/A}
=
\frac{2s}{A}
=
2s\,(12\Delta)^{1/\beta}
\le
2\cdot 12^{1/\beta}\, s^{\frac{1+\beta}{\beta}},
\end{equation*}
again using $s \ge \Delta$. Since $4\cdot 12^{(1-\beta)/\beta} \le 4\cdot 12^{1/\beta} = C_\beta$, both divergences are at most $C_\beta s^{(1+\beta)/\beta}$.
\end{proof}

\subsubsection{KL computation}
\begin{proof}[Lemma~\ref{lem:kl-total}]
The proof follows \cite{seldin2014prediction}; we include it for completeness, with the per-round KL bound adapted to the heavy-tailed family of Lemma~\ref{lem:ht-family}.
Consider the full transcript $\tau_T = (S_t, H_t, \{I_t^k\}_{k \in S_t}, \{L_t^k\}_{k \in S_t})_{t \le T}$, where $L_t^k$ is the realized loss observed for expert $k \in S_t$ at round $t$, together with the internal randomization of the meta-procedure and of the experts. Let $\tilde{\mathbb P}_\emptyset$ and $\tilde{\mathbb P}_h$ denote its law in the null game and in the $h$-th perturbed game.
By the data-processing inequality, $\mathrm{KL}(\mathbb P_\emptyset\|\mathbb P_h) \le \mathrm{KL}(\tilde{\mathbb P}_\emptyset\|\tilde{\mathbb P}_h)$, so it suffices to bound the latter.
Using the chain rule for KL divergence, the divergence decomposes into a sum over rounds of conditional divergences given the past transcript.
Conditional on any fixed past transcript, the selections $(S_t, H_t)$ and the played arms $\{I_t^k\}_{k \in S_t}$ have the same conditional distribution in both games, since they are functions of the past transcript and of the shared algorithmic randomization.
Moreover, for every expert $k \neq h$, the loss distributions are identical in the two games, so only observations of expert $h$ contribute, and only on rounds with $h \in S_t$.
Therefore
\begin{equation}
\label{eq:kl-chainrule}
\mathrm{KL}(\tilde{\mathbb P}_\emptyset\|\tilde{\mathbb P}_h)
=
\sum_{t=1}^T
\mathbb E_\emptyset\!\left[
\mathbf 1\{h \in S_t\}\;
\mathrm{KL}\bigl(\nu_{\ell^\emptyset_{h,I_t^h}} \,\big\|\, \nu_{\ell^{(h)}_{h,I_t^h}}\bigr)
\right],
\end{equation}
where $\ell^\emptyset_{h,a}$ and $\ell^{(h)}_{h,a}$ denote the mean of arm $a$ of expert $h$ in the null game and in the $h$-th perturbed game, respectively; the conditional KL divergence is well defined because the conditional law of $I_t^h$ given the past transcript is the same in both games.
By construction, $\ell^\emptyset_{h,a} = \ell^{(h)}_{h,a} + (\varepsilon+\Delta)$ for both $a \in \{1,2\}$, with $\ell^{(h)}_{h,a} - b \in \{0, \Delta\}$ and, since $\varepsilon \le \Delta$, the shift satisfies $s := \varepsilon+\Delta \in [\Delta, 2\Delta]$ and $(\ell^{(h)}_{h,a} - b) + s \le 3\Delta$.
Lemma~\ref{lem:ht-family}(ii), applied with $\theta \in \{0, \Delta\}$ and shift $s = \varepsilon+\Delta$, bounds each conditional KL divergence by $C_\beta(\varepsilon+\Delta)^{\frac{1+\beta}{\beta}}$ pointwise, hence
\begin{equation*}
\mathrm{KL}(\tilde{\mathbb P}_\emptyset\|\tilde{\mathbb P}_h)
\le
C_\beta(\varepsilon+\Delta)^{\frac{1+\beta}{\beta}}
\sum_{t=1}^T \mathbb E_\emptyset\bigl[\mathbf 1\{h \in S_t\}\bigr]
=
C_\beta(\varepsilon+\Delta)^{\frac{1+\beta}{\beta}}\,\mathbb E_\emptyset\bigl[n_h(T)\bigr].
\end{equation*}
Summing over $h \in [K]$ and using
$\sum_{h=1}^K n_h(T) = \sum_{t=1}^T |S_t| \le MT$
yields the second inequality and completes the proof.
\end{proof}

\section{Auxiliary results}
\label{sup:auxiliary}

This section collects the standing technical tools used across the proofs: the two results of \cite{de2009self} underlying the self-normalized bounds of Appendix~\ref{sec:oco_style_bounds}, Azuma's inequality, and the inversion and concentration lemmas.

For completeness, we state the two results from \cite{de2009self} used in the proof of Lemma~\ref{lemma:aux}.

\begin{theorem}[Exponential supermartingale; Theorem 9.21 in \cite{de2009self}, due to Bercu and Touati]
\label{thm:bercu-touati}
Let $\{M_n = \sum_{i=1}^n d_i\}$ be a locally square-integrable martingale adapted to a filtration $\{\mathcal F_n\}$ with $M_0 = 0$. Let $\langle M\rangle_n = \sum_{k=1}^n \mathbb E\bigl(d_k^2 \,\big|\, \mathcal F_{k-1}\bigr)$ and $[M]_n = \sum_{k=1}^n d_k^2$. Then, for all $\lambda \in \mathbb R$,
\begin{equation*}
\left\{\exp\left(\lambda M_n - \frac{\lambda^2}{2}\bigl(\langle M\rangle_n + [M]_n\bigr)\right),\; \mathcal F_n,\; n \ge 1 \right\}
\end{equation*}
is a supermartingale with mean at most $1$.
\end{theorem}
\begin{proof}
See \cite[Theorem~9.21]{de2009self}.
\end{proof}

\begin{corollary}[Method of mixtures; Corollary 12.5 in \cite{de2009self}]
\label{cor:mixture-12-5}
Let $B \ge 0$ and $A$ be two random variables satisfying the canonical assumption
$\mathbb E \exp\bigl(\lambda A - \tfrac{\lambda^2}{2} B^2\bigr) \le 1$ for all $\lambda \in \mathbb R$.
Then, for all $x \ge \sqrt{2}$ and $y > 0$,
\begin{equation*}
\mathbb P\left(\frac{|A|}{\sqrt{\left(B^2 + y\right)\left(1 + \frac{1}{2}\log\left(\frac{B^2}{y} + 1\right)\right)}} \ge x\right) \le \exp\left(-\frac{x^2}{2}\right).
\end{equation*}
\end{corollary}
\begin{proof}
See \cite[Corollary~12.5]{de2009self}.
\end{proof}

\begin{lemma}[Azuma`s inequality (from Theorem D.2 \cite{mohri2018foundations})]\label{lemma:azuma_adaptive}
Let $\{\mathcal F_t\}$ be a filtration, and let $X_t$ be a sequence of random variables adapted to $\mathcal F_t$ with $\mbE[X_t\mid \mathcal F_{t-1}]=0$ and $|X_t|\le 1$.
Then, for any fixed $n \ge 1$ and any $\delta \in (0,1)$, with probability at least $1-\delta$,
\begin{equation}
\left|\frac{1}{n}\sum_{t=1}^n X_t\right| \le \sqrt{\frac{2\log(1/\delta)}{n}}.
\end{equation}
\end{lemma}
\begin{proof}
See \cite[Theorem~D.2]{mohri2018foundations}.
\end{proof}

\begin{lemma}
\label{lemma:aux2}
Let $X_t \in [0,1]$ be a process adapted to a filtration $\mathcal F_t$, let $\mu_t := \mbE[X_t \mid \mathcal F_{t-1}]$, $S_t = \sum_{i=1}^t X_i$, $U_t = \sum_{i=1}^t \mu_i$, and fix $\x \ge 1$. On the event where \eqref{def:aux_bound_stopping} holds for all $t \ge 1$, the following inequality holds for all $t \ge 1$:
\begin{equation}
\label{eq:aux2_claim}
U_t \le S_t
+ \frac{9\x}{2}\Bigl(7+\log \x+\log(1+4S_t)\Bigr)
+ \sqrt{\,2\x(1+4S_t)\Bigl(1+\tfrac12\log(1+4S_t)\Bigr)}.
\end{equation}
\end{lemma}
\begin{proof}
Fix $t$ and write $S := S_t$, $U := U_t$ and $\Delta := U - S$. If $\Delta \le 0$, the bound \eqref{eq:aux2_claim} holds trivially, so assume $\Delta \ge 0$. Set $a := 2\x \ge 2$ and $b := 1 + 4S \ge 1$. Since $1 + S + 3U = b + 3\Delta$, the inequality \eqref{def:aux_bound_stopping} gives, with $c := \log(b + 3\Delta) \ge 0$,
\begin{equation}
\label{eq:aux2_start}
\Delta^2 \le a(b + 3\Delta)\Bigl(1 + \frac{c}{2}\Bigr)
= 3a\Bigl(1 + \frac{c}{2}\Bigr)\Delta + ab\Bigl(1 + \frac{c}{2}\Bigr).
\end{equation}
By concavity of the logarithm, $c = \log(b + 3\Delta) \le \log b + \frac{3\Delta}{b}$, hence
$ab\bigl(1 + \frac{c}{2}\bigr) \le ab\bigl(1 + \frac{\log b}{2}\bigr) + \frac{3a}{2}\Delta$, and \eqref{eq:aux2_start} yields
\begin{equation*}
\Delta^2 \le \frac{3a}{2}(3 + c)\,\Delta + ab\Bigl(1 + \frac{\log b}{2}\Bigr).
\end{equation*}
Solving this quadratic inequality in $\Delta$ (and using $\sqrt{p^2 + 4q} \le p + 2\sqrt{q}$),
\begin{equation}
\label{eq:aux2_delta}
\Delta \le \frac{3a}{2}(3 + c) + \sqrt{ab\Bigl(1 + \frac{\log b}{2}\Bigr)}.
\end{equation}
It remains to upper-bound $c$. Since $\log b \le c$, \eqref{eq:aux2_delta} implies
\begin{equation*}
b + 3\Delta \le b + \frac{9a}{2}(3 + c) + 3\sqrt{ab\Bigl(1 + \frac{c}{2}\Bigr)}
\le \left(\sqrt{b} + \sqrt{\tfrac{27a}{2}\Bigl(1 + \frac{c}{3}\Bigr)}\right)^2,
\end{equation*}
where the last step uses $\frac{9a}{2}(3+c) = \frac{27a}{2}\bigl(1 + \frac{c}{3}\bigr)$ and $3\sqrt{ab\bigl(1+\frac{c}{2}\bigr)} \le \sqrt{54\,ab\bigl(1 + \frac{c}{3}\bigr)}$ (since $6\bigl(1+\frac{c}{3}\bigr) \ge 1 + \frac{c}{2}$). Taking logarithms and factoring out $\sqrt{b}$,
\begin{equation*}
c \le \log b + 2\log\left(1 + \sqrt{\tfrac{27a}{2b}\Bigl(1 + \frac{c}{3}\Bigr)}\right)
\le \log b + 2\log\left(1 + \sqrt{\tfrac{27a}{2b}}\right) + \frac{c}{3},
\end{equation*}
where we used $1 + zw \le (1+z)w$ for $w := \sqrt{1 + c/3} \ge 1$ and $\log(1+y) \le y$. Consequently $c \le \frac{3}{2}\log b + 3\log\bigl(1 + \sqrt{27a/(2b)}\bigr)$ and, using $1 + z \le 2\sqrt{1+z^2}$,
\begin{equation*}
1 + \frac{c}{2} \le 1 + \frac{3}{4}\log b + \frac{3}{2}\log 2 + \frac{3}{4}\log\left(1 + \frac{27a}{2b}\right)
= 1 + \frac{3}{4}\log(4b + 54a).
\end{equation*}
Therefore $3 + c \le 3 + \frac{3}{2}\log(4b + 54a)$ and \eqref{eq:aux2_delta} gives
\begin{equation*}
\Delta \le \frac{9a}{4}\Bigl(2 + \log(4b + 54a)\Bigr) + \sqrt{ab\Bigl(1 + \frac{\log b}{2}\Bigr)}.
\end{equation*}
Finally, for $\x \ge 1$ we have $4b + 54a = 4 + 16S + 108\x \le 112\x(1 + 4S)$, hence
\begin{equation*}
2 + \log(4b + 54a) \le 2 + \log 112 + \log \x + \log(1 + 4S) \le 7 + \log \x + \log(1 + 4S),
\end{equation*}
since $2 + \log 112 \le 7$. Substituting $a = 2\x$ and $b = 1 + 4S$ yields \eqref{eq:aux2_claim}.
\end{proof}

\begin{lemma}[Empirical loss concentration (Lemma B.10 in \cite{shalev2014understanding})]
\label{lem:loss-eval}
Let $\mathbf{W}$ be a state space and $\mathcal{D}$ a distribution on $\mathbf{E}$.  Let $\ell:\mathbf{W}\times\mathbf{E}\to [0,1]$ be a bounded loss function. Fix a predictor $w\in\mathbf{W}$ and define its expected loss $L(w):=\mathbb{E}_{\xi\sim\mathcal{D}}[\ell(w,\xi)]$.

Then, for any $n> 0$ and $\delta\in(0,1)$, with probability at least $1-\delta$ over the draw of $\{\xi_t\}_{t=1}^n$ i.i.d.\ from $\mathcal{D}$,
\begin{equation}
\frac{1}{n}\sum_{t=1}^n \ell(w,\xi_t) - L(w)
\le
\sqrt{\tfrac{2L(w)\log(1/\delta)}{n}} + \tfrac{2\log(1/\delta)}{3n}.
\end{equation}
\end{lemma}

\begin{proof}
Follows directly from Bernstein’s inequality.
\end{proof}

\end{document}